\documentclass{tlp}

\usepackage[utf8]{inputenc}
\usepackage{amsmath}
\usepackage{amssymb}

\usepackage{yhmath}
\usepackage{stmaryrd}
\usepackage{upgreek}

\makeatletter
\def\@begintheorem#1#2{\intheoremtrue
  \normalfont\rmfamily
  \trivlist
    \pagebreak[3]\item[\hskip \labelsep{\normalfont\itshape #1\ #2}]\samepage \item[] \samepage }

\def\@opargbegintheorem#1#2#3{\intheoremtrue
  \normalfont\rmfamily
  \trivlist
    \pagebreak[3]\item[\hskip \labelsep{\normalfont\itshape #1\ #2\ \ifrembrks #3\/\global\rembrksfalse\else {\upshape(}#3\/{\upshape)}\fi}]\samepage \item[] \samepage }

\def\@endtheorem{\endtrivlist\intheoremfalse}
\makeatother

\newcommand{\head}[1]{\ensuremath{\mathit{h}(#1)}} \newcommand{\Head}[1]{\ensuremath{\mathit{H}(#1)}}
\newcommand{\body}[1]{\ensuremath{\mathit{B}(#1)}}

\newcommand{\poslits}[1]{\ensuremath{{#1}^+}}
\newcommand{\neglits}[1]{\ensuremath{{#1}^-}}

\newcommand{\pbody}[1]{\poslits{\body{#1}}}
\newcommand{\nbody}[1]{\neglits{\body{#1}}}

  \providecommand{\logfont}{\textrm}

\newcommand{\HT}{\ensuremath{\logfont{HT}}}

\newcommand{\HTC}{\ensuremath{\logfont{HT}_{\!c}}}

 \newcommand{\true}{\ensuremath{\boldsymbol{T}}}

\newcommand{\tuple}[1]{\ensuremath{\langle #1 \rangle}}

 \providecommand{\sysfont}{\textit}

\newcommand{\clingcon}{\sysfont{clingcon}}
\newcommand{\clingo}{\sysfont{clingo}}
\newcommand{\clingodl}{\clingoM{dl}}
\newcommand{\clingolp}{\clingoM{lp}}

\newcommand{\clingoM}[1]{\clingo{\small\textnormal{[}\textsc{#1}\textnormal{]}}}

\providecommand{\C}{C}

 \renewcommand{\sysfont}{\texttt}

\newtheorem{proposition}{Proposition}
\newtheorem{theorem}{Theorem}

\newtheorem{corollary}{Corollary}
\newtheorem{lemma}{Lemma}

\newtheorem{example}{Example}
\newtheorem{definition}{Definition}

\newcommand{\den}[1]{\llbracket \, #1 \, \rrbracket}

\newcommand{\comp}[1]{\wideparen{#1}} \newcommand{\denAT}[1]{\llbracket \, #1 \, \rrbracket_\AT}

\newcommand{\tprogram}{$\mathcal{T}$\nobreakdash-program}
\newcommand{\tsolution}{$\langle\AT,\TheoryAtomsES\rangle$\nobreakdash-solution}

\newcommand{\Atoms}{\ensuremath{\mathcal{A}}}
\newcommand{\TheoryAtoms}{\ensuremath{\mathcal{T}}}
\newcommand{\TheoryAtomsDN}{\ensuremath{\mathcal{F}}}
\newcommand{\TheoryAtomsES}{\ensuremath{\mathcal{E}}}
\newcommand{\htcsemanticsP}{\ensuremath{\uptau}}
\newcommand{\htcsemanticsA}{\ensuremath{\uptau}}
\newcommand{\htctsols}{\ensuremath{\Phi}}

\newcommand{\X}{\ensuremath{\mathcal{X}}}

\newcommand{\V}{\ensuremath{\mathcal{V}}}
\newcommand{\D}{\ensuremath{\mathcal{D}}}
\newcommand{\Du}{\ensuremath{\mathcal{D}_{\undefined}}}
\renewcommand{\C}{\ensuremath{\mathcal{C}}}

\newcommand{\XAT}{\ensuremath{\X_\AT}}
\newcommand{\DAT}{\ensuremath{\D_\AT}}
\newcommand{\VAT}{\ensuremath{\V_\AT}}
\newcommand{\VARSAT}[1]{\ensuremath{\varsATfunction_\AT(#1)}}

\newcommand{\AT}{\ensuremath{\mathfrak{T}}}
\newcommand{\LC}{\ensuremath{\mathfrak{L}}}

 \newcommand{\DIE}{\ensuremath{\mathfrak{DE}}}

\newcommand{\vars}[1]{\ensuremath{\mathit{vars}(#1)}}
\newcommand{\varsTfunction}[1]{\ensuremath{\mathit{vars}_{#1}}}
\newcommand{\varsT}[2]{\ensuremath{\varsTfunction{#1}(#2)}}
\newcommand{\varsATfunction}{\mathit{vars}} \newcommand{\varsAT}[1]{\mathit{vars}(#1)} 

\newcommand{\undefined}{\ensuremath{\mathbf{u}}} \newcommand{\df}[1]{\mathit{def}(#1)}
\newcommand{\eqdef}{\mathrel{\vbox{\offinterlineskip\ialign{\hfil##\hfil\cr $\scriptscriptstyle\mathrm{def}$\cr \noalign{\kern1pt}$=$\cr \noalign{\kern-0.1pt}}}}}
\newcommand{\restr}[2]{{#1|}_{\hspace{-1pt}#2}}

\newcommand{\code}[1]{\ensuremath{\mathtt{#1}}} \newcommand{\entails}{\mathbin{|\kern -.42em\approx}}
\newcommand{\nentails}{\mathbin{|\kern -.47em\approx\kern-.9em{/}}}

\newcommand{\Answer}[1]{\ensuremath{\mathit{ans}(#1)}}
\newcommand{\StableP}[2]{\ensuremath{\mathit{stb}_{#1}(#2)}}
\newcommand{\Stable}[1]{\StableP{P}{#1}}

\def\df{\mathit{def}}
\newcommand\Dom[2]{#2:#1}

\newcommand\SEM[1]{\sigma(\AT,#1)}

\def\Sat{\mathit{Sat}}
\def\Leq{\mbox{\tt <=}}
\def\Geq{\mbox{\tt >=}}
\def\Neq{\mbox{\tt !=}}
\def\Grt{\mbox{\tt >}}
\def\Low{\mbox{\tt <}}
\def\Equ{\mbox{\tt =}}
\newcommand{\Sum}[1]{\code{\&sum\{} #1 \code{\}}}

\newcommand{\prop}[1]{{#1}}

\def\anyAt{b}
\def\vregAt{a}
\def\regAt{\code{\vregAt}}
\def\thAt{\code{s}}

\def\Comp{\mathit{Comp}}
\def\dom{\mathit{dom}}
\newcommand{\uvaluation}{valuation}
\newcommand{\uvaluations}{valuations}
\newcommand{\Vu}{\ensuremath{\mathcal{V}}}
\newcommand{\DOM}[1]{\ensuremath{\delta(#1)}} \newcommand{\htctheory}{\HTC\nobreakdash-theory}

\renewcommand{\true}{\ensuremath{\mathbf{t}}}
\newcommand{\IN}{\mathbin{\hat{\in}}}
\newcommand{\map}[2]{#1 \mapsto #2}

 \newcommand{\comment}[1]{} 

\lefttitle{Pedro Cabalar et al.}

\jnlPage{1}{42}
\jnlDoiYr{20o1}
\doival{10.1017/xxxxx}

\begin{document}

\title[Strong Equivalence in ASP with Constraints]{Strong Equivalence in Answer Set Programming with Constraints}

\begin{authgrp}
\author{\sn{Pedro} \gn{Cabalar}}
\affiliation{University of Corunna, Spain}
\author{\sn{Jorge} \gn{Fandinno}}
\affiliation{University of Nebraska Omaha, USA}
\author{\sn{Torsten} \gn{Schaub}}
\affiliation{University of Potsdam, Germany}\affiliation{Potassco Solutions, Germany}
\author{\sn{Philipp} \gn{Wanko}}
\affiliation{Potassco Solutions, Germany}
\end{authgrp}

\maketitle

\begin{abstract}
  We investigates the concept of strong equivalence within the extended framework of Answer Set Programming with constraints.
Two groups of rules are considered strongly equivalent if, informally speaking, they have the same meaning in any context.
We demonstrate that, under certain assumptions, strong equivalence between rule sets in this extended setting can be precisely characterized by
  their equivalence in the logic of Here-and-There with constraints.
Furthermore, we present a translation from the language of several \clingo-based answer set solvers that handle constraints
  into the language of Here-and-There with constraints.
  This translation enables us to leverage the logic of Here-and-There to reason about strong equivalence within the context of these solvers.
We also explore the computational complexity of determining strong equivalence in this context.
\end{abstract}
 \section{Introduction}\label{sec:introduction}

Many real-world applications have a heterogeneous nature
that can only be effectively handled by combining different types of constraints.
This is commonly addressed by hybrid solving technology,
most successfully in the area of Satisfiability modulo Theories (SMT;~\citealp{niolti06a})
with neighboring areas such as Answer Set Programming (ASP;~\citealp{lifschitz08b}) following suit.
However, this increased expressiveness introduces a critical challenge:
How can we determine whether modifications to a heterogeneous specification preserve its semantic meaning,
in view of the intricate interplay of different constraint types?
This is even more severe in ASP since its nonmonotonic nature does not preserve equivalence
when substituting a part of a specification by an equivalent one.
For this, one generally needs a stronger concept of equivalence that guarantees that, informally speaking,
two expressions have the same meaning in any context.
Formally, this is referred to as \emph{strong equivalence}~\citep{lipeva01a}.
Such properties are important because they can help us simplify parts of a logic program in a modular way,
without examining its other parts.

This paper delves into the question of strong equivalence within the extended framework of Constraint Answer Set Programming (CASP;~\citealp{lierler23a}),
where the integration of linear arithmetic constraints over integers introduces new difficulties.
To address this challenge, we harness the logic of Here-and-There with constraints~(\HTC; \citealp{cakaossc16a,cafascwa20a,cafascwa20b}),
a powerful formalism that allows us to precisely characterize extensions of ASP with external theories over arbitrary domains.
By employing \HTC, we aim to develop a formal method for analyzing strong equivalence in CASP,
contributing to a deeper understanding of program optimization in this increasingly important paradigm.

As an example,
consider the following very simple ASP program with constraints in the language of the hybrid solver~\clingcon~\citep{bakaossc16a}:
\begin{align*}
    &\texttt{:- not a, \&sum\{s\} >= 120.}
    \\
    &\texttt{a :- \&sum\{s\} > 100.}
\end{align*}
The Boolean atom~$\texttt{a}$ is true when an alarm is set in the car and
$\texttt{s}$ is an integer variable representing the speed of the car.
The first rule tells us that if the alarm is not set, we cannot choose speeds exceeding 120~km/h.
The second one states that if the speed is greater than 100~km/h, then the alarm must be set.
We are interested in answering the question of whether the first rule is redundant in the presence of the second rule.
In terms of strong equivalence, if the pair of rules together are strongly equivalent to the second rule alone.
As we discuss in Section~\ref{sec:strongeq}, this is indeed the case.

Our main result is a characterization of strong equivalence in the context of ASP with constraints,
given in Section~\ref{sec:stronge.lp}.
The proof of this result is based on two intermediate results that are of independent interest:
A characterization of strong equivalence in the context of \HTC\ developed in Section~\ref{sec:strongeq} and
an answer set preserving transformation of logic programs with constraints into~\HTC\nobreakdash-theories
given in Section~\ref{sec:lpabstract}.
We also study the computational complexity of deciding strong equivalence in this context in Section~\ref{sec:complexity}.

 \section{The Logic of Here-and-there with Constraints}\label{sec:htc}
\providecommand{\C}{\TheoryAtoms}
The logic of \emph{Here-and-There with constraints} (\HTC) along with its equilibrium models
provides logical foundations for constraint satisfaction problems (CSPs) in the setting of ASP.
Its approach follows the one of the traditional Boolean
Logic of Here-and-There~\citep{heyting30a} and its nonmonotonic extension of Equilibrium Logic~\citep{pearce96a}.
In \HTC, a CSP is expressed as a triple $\tuple{\X,\D,\C}$, also called \emph{signature},
where
\X\ is a set of \emph{variables} over some non-empty \emph{domain} \D\
and \C\ is a set of \emph{constraint atoms}.
A constraint atom provides an abstract way to relate values of variables and constants
according to the atom's semantics
(just as in lazy SMT~\citep{niolti06a}).
Most useful constraint atoms have a structured syntax, but in the general case, we may simply consider them as strings.
For instance, linear equations are expressions of the form ``$x+y=4$'',
where $x$ and $y$ are variables from~\X\ and $4$ is a constant representing some element from \D.

Variables can be assigned some value from $\D$ or left \emph{undefined}.
For the latter, we use the special symbol $\undefined \notin \D$ and the extended domain $\Du \eqdef \D \cup \{\undefined\}$.
The set $\mathit{vars}(c) \subseteq \X$
collects all variables occurring in constraint atom $c$.
We assume that every constraint atom~$c$ satisfies~$\vars{c} \neq \emptyset$
(otherwise, it is just a truth constant).

A \emph{\uvaluation} $v$ over $\X,\D$ is a function $v:\X\rightarrow\Du$.
We let~$\Vu^{\X,\D}$, or simply $\Vu$ when clear from the context, stand for the set of all \uvaluations\ over~$\X,\D$.
Moreover, \uvaluation\ $v|_X: X\rightarrow\Du $ is the projection of $v$ on $X \subseteq \X$.
Accordingly, for any set of valuations $V\subseteq\Vu$,
we define their restriction to $X$ as $\restr{V}{X} \eqdef \{\restr{v}{X} \mid v \in V\}$.
A valuation $v$ can be viewed as a set of pairs of the form
\(
\{ \map{x}{v(x)} \mid x \in \X, v(x)\in\D\}
\),
which drops any pairs $\map{x}{\undefined}$ for $x\in\X$.
This allows us to use standard set inclusion for comparison.
In view of this, $v\subseteq v'$ stands for
\begin{align*}
  \{ \map{x}{v(x)} \mid x \in \X, v (x)\in\D\}
  \ \subseteq \
  \{ \map{x}{v'(x)} \mid x \in \X, v'(x)\in\D\}.
\end{align*}
This is equivalent to: $v(x) \in \D$ implies $v'(x)=v(x)$ for all $x \in \X$.
We define the \emph{domain} of a valuation $v$ as the set of variables that have a defined value,
namely,
$$\dom(v) \eqdef \{x \in \X \mid v(x)\neq \undefined\}.$$
We also allow for applying a \uvaluation~$v$ to fixed domain values,
and so extend their type to
\(
v:\X\cup \Du\rightarrow\Du
\)
by fixing $v(d) = d$ for any $d \in \D_\undefined$.

The semantics of constraint atoms is defined in \HTC\ via \emph{denotations},
which are functions of form
\(
\den{\cdot}:\C\rightarrow 2^{\Vu}
\),
mapping each constraint atom to a set of \uvaluations.
Following~\citep{cafascwa20a}, we assume in the sequel that they satisfy the following properties
for all
$c\in\C$,
$x\in\X$, and
${v,v' \in \Vu}$:
\begin{enumerate}
\item $v \in \den{c}$ and $v \subseteq v'$ imply $v' \in \den{c}$,
  \label{den:prt:0}
\item $v \in \den{c}$ implies $v \in \den{c[x/v(x)]}$,
  \label{den:prt:1}
\item if $v(x)=v'(x)$ for all $x \in \mathit{vars}(c)$ then $v \in \den{c}$ iff $v' \in \den{c}$.
  \label{den:prt:2}
\end{enumerate}
where $c[s/s']$ is the syntactic replacement in $c$ of subexpression~$s$ by~$s'$.
We also assume that~$c[x/d] \in \C$ for any constraint atom~$c[x] \in \C$, variable~$x \in \X$ and~$d \in \Du$.
That is, replacing a variable by any element of the extended domain results in a syntactically valid constraint atom.
Intuitively,
Condition~\ref{den:prt:0} makes constraint atoms behave monotonically with respect to definedness.
Condition~\ref{den:prt:1} stipulates that denotations respect the role of variables as placeholders for values, that is,
replacing variables by their assigned values does not change how an expression is evaluated.
Condition~\ref{den:prt:2} asserts that the denotation of $c$ is fixed by combinations of values for $\vars{c}$;
other variables may freely vary.

A formula $\varphi$ over signature $\tuple{\X,\D,\C}$ is defined as
\begin{align*}
  \varphi::= \bot \mid c\mid \varphi \land \varphi \mid  \varphi \lor \varphi \mid  \varphi \rightarrow \varphi \quad\text{ where }c\in\C.
\end{align*}
We define $\top$ as $\bot \rightarrow \bot$ and $\neg\varphi$ as $\varphi \rightarrow \bot$ for any formula~$\varphi$.
We let $\vars{\varphi}$ stand for the set of variables in \X\ occurring in all constraint atoms in formula~$\varphi$.
An \emph{\htctheory} is a set of formulas.

In \HTC,
an \emph{interpretation} over $\X,\D$ is a pair $\langle h,t \rangle$
of \uvaluations\ over $\X,\D$ such that $h\subseteq t$.
The interpretation is \emph{total} if $h=t$.
Given a denotation $\den{\cdot}$,
  an interpretation $\langle h,t \rangle$ \emph{satisfies} a formula~$\varphi$,
  written $\langle h,t \rangle \models \varphi$,
  if \begin{enumerate}
\item $\langle h,t \rangle \models c \text{ if } h\in \den{c}$ \label{item:htc:atom}
  \item $\langle h,t\rangle \models \varphi \land \psi \text{ if }  \langle h,t\rangle \models \varphi \text{ and }  \langle h,t\rangle \models \psi$
  \item $\langle h,t\rangle \models \varphi \lor \psi \text{ if }  \langle h,t\rangle \models \varphi \text{ or }  \langle h,t\rangle \models \psi$
  \item $\langle h,t\rangle \models \varphi \rightarrow \psi
    \text{ if }\langle w,t\rangle \not\models \varphi \text{ or }\langle w,t\rangle \models \psi
    \text{ for }w\in\{h,t\}$
  \end{enumerate}
We say that an interpretation~$\tuple{h,t}$ is a model of a theory~$\Gamma$,
written $\tuple{h,t} \models \Gamma$,
when $\tuple{h,t} \models \varphi$ for every $\varphi \in \Gamma$.
We write~$\Gamma \models \Gamma'$ when every model of~$\Gamma$ is also a model of~$\Gamma'$.
We write $\Gamma \equiv \Gamma'$ if $\Gamma$ and $\Gamma'$ have the same models.
We omit braces whenever a theory is a singleton.

\begin{proposition}[\mbox{\citealp[Proposition 3]{cakaossc16a}}]\label{prop:htc.basic}
For any formula $\varphi$, we have
\begin{enumerate}
\item $\tuple{h,t} \models \varphi$ implies $\tuple{t,t} \models \varphi$,
\item $\tuple{h,t} \models \neg \varphi$ iff $\tuple{t,t} \not\models \varphi$, and
\item any tautology in \HT\ is also a tautology in \HTC.
\end{enumerate}
\end{proposition}
The first item reflects the well known Persistence property in constructive logics.
The second one tells us that negation is only evaluated in the there world.
The third item allows us to derive the strong equivalence of two expressions in \HTC\ from their equivalence in \HT\
when treating constraint atoms monolithically, though more equivalences can usually be derived using~\HTC.

\HTC\ makes few assumptions about the syntactic form or the semantics of constraint atoms.
In the current paper, however, we introduce a specific kind of constraint atom that is useful later on for some of the formal results.
Given a subset $\D'\subseteq\D$, we define the associated\footnote{Technically, $\D'$ must not be thought as an element of the atom syntax but as part of the atom name,
  i.e.\ we have a different atom for each $\D'\subseteq\D$).}
constraint atom ${\Dom{\D'}{x}}$ with denotation:
\begin{eqnarray*}
\den{\Dom{\D'}{x}} \ \eqdef \ \{ v \in \Vu^{\X,\D} \mid v(x) \in \D'\}
\end{eqnarray*}
and~$\vars{\Dom{\D'}{x}}=\{x\}$,
that is, ${\Dom{\D'}{x}}$ asserts that $x$ has some value in subdomain~$\D'$.
Note that ${v(x) \in \D'}$ implies that $x$ is defined in $v$, since ${\undefined \not\in \D \supseteq \D'}$.
We use the abbreviation $\df(x)$ to stand for $\Dom{\D}{x}$, so that this atom holds iff $x$ has some value,
i.e., it is not undefined, or $v(x) \neq \undefined$.

The nonmonotonic extension of \HTC\ is defined in terms of equilibrium models,
being minimal models in \HTC\ in the following sense.
\begin{definition}[Equilibrium/Stable model]\label{def:eqmodel}
A (total) interpretation $\langle t,t\rangle$ is an \emph{equilibrium model} of a theory~$\Gamma$,
if $\langle t,t\rangle \models \Gamma$ and there is no $h\subset t$
such that $\langle h,t\rangle \models \Gamma$.

If~$\tuple{t,t}$ is an equilibrium model of~$\Gamma$,
then we say that~$t$ is a \emph{stable model} of~$\Gamma$.
\end{definition}

As detailed by~\cite{cakaossc16a},
the original logic of Here-and-There can be obtained as a special case of \HTC\ with a signature $\tuple{\X,\D,\C}$,
where $\X=\Atoms$ represents logical propositions, the domain $\D=\{\mathbf{t}\}$ contains a unique value (standing for ``true'') and the set of constraint atoms $\C=\Atoms$ coincides with the set of propositions.
In this way, each logical proposition $\regAt$ becomes both a constraint atom $\regAt \in \C$ and a homonymous variable $\vregAt \in \X$ so that $\vars{\regAt}=\{\vregAt\}$ (we use different fonts for the same name to stress the different role).
The denotation of a regular atom is fixed to
\begin{gather*}
  \den{\regAt} \eqdef \den{\Dom{\{\mathbf{t}\}}{\vregAt}} = \{v \in \Vu^{\X,\D} \mid v(\vregAt)=\mathbf{t} \}
\end{gather*}
Moreover, we can establish a one-to-one correspondence between any propositional interpretation,
represented as a set $X$ of regular atoms, and the valuation $v$ that assigns~$\mathbf{t}$ to all members of~$X$,
i.e., $X=\{\regAt \mid v(\vregAt)=\mathbf{t} \}$.
Note that a false atom, viz.\ $\regAt \not\in X$, is actually undefined in the valuation, viz.\ $v(\vregAt)=\undefined$,
since having no value is the default situation.\footnote{Allowing the assignment of a second truth value $\mathbf{f} \in \D$ for variable $\vregAt$
  would rather correspond to the explicit negation $-\regAt$ of the regular atom.}
Once this correspondence is established,
it is easy to see that the definition of equilibrium and stable models in Definition~\ref{def:eqmodel}
collapses to their standard definition for propositional theories (and also, logic programs in ASP).

Treating logical propositions as constraint atoms allows us not only to capture standard ASP
but also to combine regular atoms with other constraint atoms in a homogeneous way,
even when we deal with a larger domain $\D \supset \{\mathbf{t}\}$ such as, for instance, $\D = \mathbb{Z} \cup \{\mathbf{t}\}$.
In this setting, we may observe that while $\den{\regAt} \subseteq \den{\df(\vregAt)}$
(i.e., if the atom has value~$\mathbf{t}$, it is obviously defined),
the opposite $\den{\df(\vregAt)} \subseteq \den{\regAt}$ does not necessarily hold.
For instance, we may have now some valuation $v(\vregAt)=7$, and so $v \in \den{\df(\vregAt)}$, but $v \not\in \den{\regAt}$.
For this reason, we assume the inclusion of the following axiom
\begin{eqnarray}
\df(\vregAt) \to \regAt \label{f:booldomain}
\end{eqnarray}
to enforce $\den{\df(\vregAt)}=\den{\regAt}$ for each regular atom $\regAt$.
This forces regular atoms to be either true or undefined.
It does not constrain non\nobreakdash-regular atoms from taking any value in the domain.

Finally, one more possibility we may consider in \HTC\ is treating all constraint atoms as regular atoms,
so that we do not inspect their meaning in terms of an external theory but only consider their truth as propositions.
\begin{definition}[Regular stable model]\label{def:regularsm}
Let $\Gamma$ be an \HTC\ theory over signature $\tuple{\X,\D,\C}$.
A set of atoms $X \subseteq \C$ is a \emph{regular stable model} of $\Gamma$
if the valuation $t=\{\map{\regAt}{\mathbf{t}} \mid \regAt \in X\}$ is a stable model of $\Gamma$ over signature $\tuple{\C,\{\mathbf{t}\},\C}$
while fixing the denotation $\den{\regAt} \eqdef \den{\Dom{\{\mathbf{t}\}}{\vregAt}}$ for every $\regAt \in \C$.
\end{definition}
In other words, regular stable models are the result of considering an \HTC-theory $\Gamma$ as a propositional \HT-theory
where constraint atoms are treated as logical propositions.
This definition is useful in defining the semantics of logic programs with constraints according to \clingo\ (see Section~\ref{sec:clingo5}).

 \section{Strong Equivalence in \HTC}\label{sec:strongeq}

One of the most important applications of \HT\ is its use in ASP for equivalent transformations among different programs or fragments of programs.
In general, if we want to safely replace program $P$ by $Q$, it does not suffice to check that their sets of stable models coincide, because the semantics of ASP programs cannot be figured out by looking at single rules in isolation.
We can easily extrapolate this concept to \HTC\ as follows:
\begin{definition}[Strong equivalence]
\HTC-Theories $\Gamma$ and $\Gamma'$ are \emph{strongly equivalent} when $\Gamma \cup \Delta$ and $\Gamma' \cup \Delta$ have the same stable models for any arbitrary \HTC-theory~$\Delta$.
\end{definition}
Theory $\Delta$ is sometimes called a \emph{context}, so that strong equivalence guarantees that $\Gamma$ can be replaced by $\Gamma'$ in any context (and vice versa).
An important property of \HT\ proved by~\citeN{lipeva01a} is that two regular programs $P$ and $Q$ are strongly equivalent
if and only if they are equivalent in \HT.

Accordingly, we may wonder whether a similar result holds for \HTC.
The following theorem is an immediate result:
the equivalence  $\Gamma \equiv \Gamma'$ in \HTC\ implies that \HTC-theories $\Gamma$ and $\Gamma'$ are strongly equivalent.
\begin{theorem}\label{th:strongeq-suf}
  If \HTC-theories $\Gamma$ and $\Gamma'$ are equivalent in \HTC, that is, $\Gamma \equiv \Gamma'$,
  then $\Gamma$ and $\Gamma'$ are strongly equivalent.
\end{theorem}
\begin{proof}
Assume~$\Gamma \equiv \Gamma'$ and take any arbitrary theory~$\Delta$.
Then,
$\Gamma$ and~$\Gamma'$ have the same models and, as a result, $\Gamma \cup \Delta$ and~$\Gamma' \cup \Delta$ have also the same models.
Since equilibrium models are a selection among \HTC\ models, their equilibrium and stable models also coincide.\end{proof}

For illustration,
let us use Theorem~\ref{th:strongeq-suf} to prove the strong equivalence of two simple \HTC\nobreakdash-theories.
Consider the following two formulas,
similar to the ones from the introductory section:
\begin{align}
    &\bot\leftarrow \neg \regAt \wedge s \geq 120
    \label{eq:ex.htc.contraint}
    \\
    &\regAt \leftarrow s > 100
    \label{eq:ex.htc.rule}
\end{align}
In fact, we show in Section~\ref{sec:htcsemantics} that
these two formulas are an essential part of the \HTC-based translation of the logic program mentioned in the introduction.
Let us show that~$\Gamma = \{ \eqref{eq:ex.htc.contraint}, \eqref{eq:ex.htc.rule}\}$ is strongly equivalent to~$\Gamma' = \{ \eqref{eq:ex.htc.rule}\}$.
By Theorem~\ref{th:strongeq-suf}, it suffices to show that~$\Gamma \equiv \Gamma'$ and, since~$\Gamma' \subseteq \Gamma$, it suffices to show that~$\eqref{eq:ex.htc.rule} \to \eqref{eq:ex.htc.contraint}$ is an \HTC\nobreakdash-tautology.
The following property is useful to prove this result.

Let $\varphi[c/\psi]$ be the uniform replacement of a constraint atom $c$ occurring in formula $\varphi$ by a formula $\psi$.
We show below that \HTC\ satisfies the rule of uniform substitution.
\begin{proposition}\label{prop:uniform-subs}
If $\varphi[c]$ is an \HTC\ tautology then $\varphi[c/\psi]$ is an \HTC-tautology.
\end{proposition}
\begin{proof}
As a proof sketch, simply observe that if the original formula $\varphi[c]$ is a tautology, it is satisfied for any possible combination of satisfactions for $c$ in $h$ and in $t$.
Each time we replace $c$ by some formula $\psi$ in a uniform way for some interpretation $\tuple{h,t}$, this corresponds to one of these possible truth combinations for atom $c$, and so $\varphi[c/\psi]$ is also satisfied.
\end{proof}

Let is now resume our example.
The \HT-tautology
\begin{gather*}
(\alpha \to \beta) \to ((\beta \to \gamma) \to \neg(\neg \gamma \wedge \alpha))
\end{gather*}
is also an \HTC-tautology by Proposition~\ref{prop:htc.basic}.
By applying Proposition~\ref{prop:uniform-subs} to this \HTC\nobreakdash-tautology with substitution~${\alpha \mapsto s\geq 120}$, ${\beta \mapsto s > 100}$, and ${\gamma\mapsto\regAt}$, we obtain the following \HTC\nobreakdash-tautology:
\begin{align*}
(s\geq 120 \to s>100)
\to(\eqref{eq:ex.htc.rule} \to \eqref{eq:ex.htc.contraint})
\end{align*}
To show that~${\eqref{eq:ex.htc.rule} \to \eqref{eq:ex.htc.contraint}}$ is an \HTC\nobreakdash-tautology and, thus, $\Gamma$ and~$\Gamma'$ are strongly equivalent, it suffices to show that the antecedent of this implication is an \HTC\nobreakdash-tautology.
This immediately follows once we assume the usual semantics of linear inequalities and, thus, that our denotation satisfies
\begin{align*}
  \den{s \geq 120} \subseteq \den{s>100} .
\end{align*}
As expected,
the result may not hold if we assume that these two constraint atoms have a different meaning than the one they have for linear inequalities.

This illustrates how we can use Theorem~\ref{th:strongeq-suf} to prove the strong equivalence of two \HTC-theories.
To prove that \HTC\nobreakdash-equivalence is also a necessary condition for strong equivalence,
we depend on the form of the context theory $\Delta$.
For instance, in the case in which $\Gamma$ and $\Gamma'$ are regular ASP programs, it is well-known that if we restrict the form of $\Delta$ to sets of facts, we obtain a weaker concept called \emph{uniform equivalence}, that may hold even if $\Gamma$ and $\Gamma'$ do not have the same \HT\ models.
In the case of \HTC, the variability in the possible context theory $\Delta$ is much higher since the syntax and semantics of constraint atoms are very general, and few assumptions can be made on them.
Yet, if our theory accepts at least a constraint atom $\df(x)$ in \C\ for each $x \in \X$,
then we can use this construct (in the context theory $\Delta$) to prove the other direction of the strong equivalence characterization.
\begin{theorem}\label{th:strongeq-nec}
  If \HTC-theories $\Gamma$ and $\Gamma'$ are strongly equivalent,
  then they are equivalent in \HTC, that is, $\Gamma \equiv \Gamma'$.
\end{theorem}
\begin{proof}
We proceed by contraposition, that is, proving that $\Gamma \not\equiv \Gamma'$ implies that $\Gamma$ and $\Gamma'$ are not strongly equivalent.
To this aim, we build some theory $\Delta$ such that $\Gamma \cup \Delta$ and $\Gamma' \cup \Delta$ have different stable models.
Without loss of generality, suppose there exists some model $\tuple{h,t} \models \Gamma$ but $\tuple{h,t} \not\models \Gamma'$.
Note that, by persistence, $\tuple{t,t} \models \Gamma$, but we do not know whether $\tuple{t,t} \models \Gamma'$ or not, so we separate the proof in two cases.

\noindent{\em Case 1}. Suppose first that $\tuple{t,t} \not\models \Gamma'$.
Let us build the theory
\begin{eqnarray*}
\Delta=\{\df(x) \mid x \in \X, t(x)\neq \undefined\}
\end{eqnarray*}
that exclusively consists of constraint atoms.
We can easily see that $\tuple{t,t} \models \Delta$ because $\Delta$ collects precisely those $\df(x)$ for which $t(x) \neq\undefined$ and so, $t \in \den{\df(x)}$ trivially.
As a result, $\tuple{t,t} \models \Gamma \cup \Delta$ and, to prove that $\tuple{t,t}$ is in equilibrium, suppose we have a smaller $v \subset t$ satisfying $\Gamma \cup \Delta$.
Then there is some variable $x \in \X$ for which $v(x)=\undefined$ whereas ${t(x) \neq \undefined}$.
The former implies $\tuple{v,t} \not\models \df(x)$ while the latter implies $\df(x) \in \Delta$, so we conclude $\tuple{v,t} \not\models \Delta$ reaching a contradiction.

\noindent {\em Case 2}.
Suppose that $\tuple{t,t} \models \Gamma'$.
Take the theory $\Delta = \Delta_1 \cup \Delta_2$ with:
\begin{eqnarray*}
\Delta_1 & = & \{\df(x) \mid x \in \X, \, h(x) \neq \undefined \}
\end{eqnarray*}
and $\Delta_2$ consisting of all rules $\df(x) \leftarrow \df(y)$ for all pair variables $x,y$ in the set:
\[
\{z \in \X \mid h(z)= \undefined, t(z)\neq \undefined \}
\]
We prove first that $\tuple{t,t}$ is not an equilibrium model for $\Gamma \cup \Delta$ because $\tuple{h,t}$ is, indeed, a model of this theory.
To show this, it suffices to see that $\tuple{h,t} \models \Delta$.
First, it follows that~$\tuple{h,t} \models \Delta_1$ because~$\Delta_1$ only contains facts of the form~$\df(x)$ per each variable $x$ satisfying ${h(x) \neq \undefined}$.
Second, ${\tuple{h,t} \models \Delta_2}$ also follows because for all the implications of the form of~$\df(x) \leftarrow \df(y)$ in $\Delta_2$, both $\tuple{h,t} \not\models \df(y)$---because $h(y)=\undefined$---and $\tuple{t,t} \models \df(x)$---because $t(x) \neq \undefined$.
It remains to be proven that $\tuple{t,t}$ is an equilibrium model of $\Gamma' \cup \Delta$.
We can see that $\tuple{t,t} \models \Delta$ follows by persistence because~$\tuple{h,t} \models \Delta$, and thus, $\tuple{t,t} \models \Gamma' \cup \Delta$.
Suppose, by the sake of contradiction, that there existed some $v \subset t$ such that $\tuple{v,t} \models \Gamma' \cup \Delta$.
Since $\tuple{v,t} \models \Delta_1$, any variable $x$ defined in $h$ must be defined in $v$ as well, but as $v \subset t$, $v(x)=t(x)=h(x)$ and so, we conclude $h \subseteq v$.
However, we also know $\tuple{h,t} \not\models \Gamma' \subseteq \Gamma' \cup \Delta$, and so $v \neq h$, that is, $h \subset v \subset t$.
Consider any variable $y$ defined in $v$ but not in $h$, that is, $h(y)=\undefined$ and $\undefined \neq v(y) = t(y)$ by persistence.
Now, $y$ is undefined in $h$ and defined in $t$, and suppose we take any other variable $x$ in the same situation.
Then, we have an implication $\df(x) \leftarrow \df(y)$ in $\Delta_2$ that must be satisfied by $\tuple{v,t}$ whereas, on the other hand, $\tuple{v,t} \models \df(y)$ because $v(y) \neq \undefined$.
As a result, $\tuple{v,t} \models \df(x)$ for all variables $x$ undefined in $h$ but not in $t$.
But as $\tuple{v,t} \models \Delta_1$ the same happens for variables defined in $h$.
As a result, $\tuple{v,t} \models \df(x)$ iff $\tuple{t,t} \models \df(x)$ and this implies $v(x)=t(x)$ for all variables, namely, $v=t$ reaching a contradiction.
Therefore, $\tuple{t,t}$ is an equilibrium model of $\Gamma' \cup \Delta$.
\end{proof}
This proof is analogous to the original one for \HT\ by~\citeN{lipeva01a}, using here constraint atoms $\df(x)$ to play
the role of regular atoms in the original proof.
 
\section{Logic Programs with Abstract Theories as \HTC-Theories}\label{sec:lpabstract}

This section details the development of a transformation from logic programs with constraints into \HTC-theories.
We begin by reviewing the language of logic programs with constraints employed by the ASP solver
\clingo~\sysfont{5} (\citealt{gekakaosscwa16a}; see also Section~\ref{sec:clingo5}).
Subsequently, we introduce a novel, simplified semantic definition for this language,
demonstrating its equivalence to the original definition under widely accepted assumptions
that are inherent to most hybrid solvers (Section~\ref{sec:simpler}).
This streamlined definition
not only enhances the clarity of the technical development presented in the subsequent sections
but also holds the potential to facilitate future advancements in hybrid solver design.
Building upon this foundation,
we revisit the definition of answer sets for logic programs with constraints,
as recently formalized by~\cite{cafascwa23a}.
Finally,
we present the central contribution of this work:
a transformation of logic programs with constraints into \HTC-theories, and
formally prove the preservation of answer sets under this transformation
(Section~\ref{sec:htcsemantics}).
\subsection{Theory solving in {\rm\clingo}}\label{sec:clingo5}

The main distinctive feature of \clingo~\sysfont{5} is the introduction of theory atoms in its syntax.
We review its most recent semantic characterization based on the concept of \emph{abstract theories}~\citep{cafascwa23a}.
We consider an alphabet consisting of two disjoint sets, namely,
a set~\Atoms\ of \emph{propositional} (or \emph{regular}) \emph{atoms}
and a set \TheoryAtoms\ of \emph{theory atoms}, whose truth is governed by some external theory.
We use letters $\regAt$, $\thAt$, and $\anyAt$ and\comment{T: fonts\dots? Cf last section\par $a$, $t$, $\mathit{at}$\dots?}
variants of them for atoms in $\Atoms$, $\TheoryAtoms,$ and $\Atoms \cup \TheoryAtoms$, respectively.
In \clingo~\texttt{5}, theory atoms are expressions preceded by `\code{\&}', but their internal syntax is not predetermined: it can be defined by the user to build new extensions.
As an example, the system \clingcon\ extends the input language of \clingo\ with linear equations, represented as theory atoms of the form
\begin{align}\label{clingcon:linear:constraint}
  \Sum{\mathit{k_1*x_1};\dots;\mathit{k_n*x_n}} \prec k_0
\end{align}
where each $x_i$ is an integer variable and each $k_i\in\mathbb{Z}$ an integer constant for $0\leq i\leq n$,
whereas $\prec$ is a comparison symbol such as \Leq, \Equ, \Neq, \Low, \Grt, \Geq.
Several theory atoms may represent the same theory entity.
For instance, $\Sum{x} \Grt 0$  and $\code{\&sum\{x\}}\Geq 1$ actually represent the same condition (as linear equations).

A \emph{literal} is any atom $b \in \Atoms \cup \TheoryAtoms$ or its default negation $\neg b$.

A \tprogram\ over $\langle\Atoms,\TheoryAtoms\rangle$ is a set of rules of the form
\begin{align}
    \anyAt_0 \leftarrow \anyAt_1,\dots,\anyAt_n,\neg \anyAt_{n+1},\dots,\neg \anyAt_m \label{theory:rule}
\end{align}
where $\anyAt_i\in\mathcal{A}\cup \TheoryAtoms$ for $1 \leq i \leq m$ and~$\anyAt_0\in \mathcal{A}\cup \TheoryAtoms \cup \{ \bot \}$ with $\bot \not\in \Atoms \cup \TheoryAtoms$ denoting the falsum constant.
We sometimes identify \eqref{theory:rule} with the formula
$$\anyAt_1 \wedge \dots \wedge \anyAt_n \wedge \neg \anyAt_{n+1} \wedge \dots \wedge \neg \anyAt_m \to \anyAt_0.$$
We let notations ${\head{r}\eqdef b_0}$, ${\pbody{r}\eqdef \{\anyAt_1,\dots,\anyAt_n\}}$ and
${\nbody{r}\eqdef\{\anyAt_{n+1},\dots,\anyAt_m\}}$ stand for the \emph{head}, the \emph{positive} and the \emph{negative body atoms} of a rule~$r$ as in~\eqref{theory:rule}.
The set of \emph{body atoms} of $r$ is just ${\body{r} \eqdef \pbody{r} \cup \nbody{r}}$.
Finally, the sets of \emph{body} and \emph{head atoms} of a program $P$ are defined as
${\body{P}\eqdef \bigcup_{r \in P} \body{r}}$ and ${\Head{P}\eqdef \{\head{r}\mid r\in P\}\setminus \{\bot\}}$, respectively.

The semantics of~\tprogram{s} in \clingo~\citep{gekakaosscwa16a}, relies on a two\nobreakdash-step process:
(1) generate regular stable models (as in Definition~\ref{def:regularsm}) and
(2) select the ones passing a theory certification.
We present next this semantics following the formalization recently introduced by~\citeauthor{cafascwa23a} in~\citeyearpar{cafascwa23a}.

An \emph{abstract theory}~\AT\ is a triple
\(
\langle \TheoryAtoms, \Sat, \comp{\cdot} \,\rangle
\)
where~\TheoryAtoms\ is the set of \emph{theory atoms} constituting the language of the abstract theory,
$\Sat \subseteq 2^\TheoryAtoms$ is the set of \AT-\emph{satisfiable} sets of theory atoms, and
$\comp{\cdot}:\TheoryAtoms \rightarrow\TheoryAtoms$ is a function mapping theory atoms to their \emph{complement}
such that $\comp{\comp{\thAt}}=\thAt$ for any $\thAt \in\TheoryAtoms$.
We define $\comp{S} = \{\comp{\thAt}\mid\thAt\in S\}$ for any set~$S\subseteq\TheoryAtoms$.
\comment{T: What about using $T$ rather than $S$\dots?}

We partition the set of theory atoms into two disjoint sets, namely, a set~$\TheoryAtomsES$ of \emph{external} theory atoms and a set~$\TheoryAtomsDN$ of \emph{founded} theory atoms.
Intuitively, the truth of each external atom in \TheoryAtomsES\ requires no justification.
Founded atoms on the other hand must be derived by the \tprogram.
We assume that founded atoms do not occur in the body of rules.
We refer to the work by~\cite{jakaosscscwa17a} for a justification for this assumption.

Given any set $S$ of theory atoms, we define its \emph{(complemented) completion} with respect to external atoms~\TheoryAtomsES, denoted by~$\Comp_{\TheoryAtomsES}(S)$, as:
\begin{eqnarray*}
\Comp_{\TheoryAtomsES}(S)
  \eqdef S \ \cup \ (\comp{\TheoryAtomsES\setminus S})
\end{eqnarray*}
In other words, we add the complement atom~$\comp{\thAt}$ for every external atom~$\thAt$ that does not occur explicitly in $S$.
\begin{definition}[Solution; \citealp{cafascwa23a}]
    Given a theory~$\AT=\langle \TheoryAtoms, \Sat, \comp{\cdot} \,\rangle$ and
    a set~${\TheoryAtomsES \subseteq \TheoryAtoms}$ of external theory atoms,
    a set~${S \subseteq \TheoryAtoms}$ of theory atoms  is a \emph{$\langle\AT,\TheoryAtomsES\rangle$-solution},
    if $S\in\Sat$ and $\Comp_{\TheoryAtomsES}(S)\in\Sat$.
\end{definition}
That is, $S$ is a \emph{$\langle\AT,\TheoryAtomsES\rangle$-solution} whenever
both~$S$ and~$\Comp_{\TheoryAtomsES}(S)$ are \AT-satisfiable.
\begin{definition}[Theory stable model\/\footnotemark; \citealp{cafascwa23a}]\label{def:stable.model}\nobreak Given a theory~$\AT=\langle \TheoryAtoms, \Sat, \comp{\cdot} \,\rangle$ and a set~${\TheoryAtomsES \subseteq \TheoryAtoms}$ of external theory atoms,
a set $X\subseteq \Atoms\cup\TheoryAtoms$ of atoms is a $\tuple{\AT,\TheoryAtomsES}$\nobreakdash-\emph{stable model} of a \tprogram\ $P$,
if there is some \tsolution\ $S$ such that $X$ is a regular stable model
of the program
\begin{align}\label{eq:transformation}
  P
  \cup
  \{ {\thAt\leftarrow      } \mid \thAt\in (S\cap\TheoryAtomsES) \}
  \cup
  \{ {\bot \leftarrow \thAt} \mid \thAt\in ((\TheoryAtoms\cap\Head{P}) \setminus S) \}
  \ .
\end{align}
\end{definition}
\footnotetext{The only minor difference to the original definition of a $\tuple{\AT,\TheoryAtomsES}$\nobreakdash-stable model
  by~\cite{cafascwa23a} is the characterization of regular stable models.
  While~\cite{cafascwa23a} use the program reduct~\citep{gellif88b} we resort in Definition~\ref{def:regularsm} to \HT\ instead.}

As an example of an abstract theory,
consider the case of \emph{linear equations} $\LC$ which can be used to capture \clingcon-programs.
This theory is formally defined as $\LC=\langle\TheoryAtoms,\Sat,\comp{\cdot}\,\rangle$,
where
\begin{itemize}
\item \TheoryAtoms\ is the set of all expressions of form \eqref{clingcon:linear:constraint},
\item $\Sat$ is the set of all subsets $S\subseteq\TheoryAtoms$ of expressions of form~\eqref{clingcon:linear:constraint}
  for which there exists an assignment of integer values to their variables that
  satisfies all linear equations in $S$ according to their usual meaning, and
\item the complement function $\comp{\code{\&sum\{\cdot\}} \prec c} $ is defined as
$\code{\&sum\{\cdot\}}\mathrel{\comp{\prec}}c$
  with~$\comp{\prec}$ defined according to the following table:
\begin{center}
    \renewcommand{\arraystretch}{1.4}
    \begin{tabular}[1cm]{
       |  @{\hskip5pt}c@{\hskip10pt}
       || @{\hskip5pt}c@{\hskip10pt}
       |  @{\hskip5pt}c@{\hskip10pt}
       |  @{\hskip5pt}c@{\hskip10pt}
       |  @{\hskip5pt}c@{\hskip10pt}
       |  @{\hskip5pt}c@{\hskip10pt}
       |  @{\hskip5pt}c@{\hskip10pt}
       |}
      \cline{1-7}
       $\prec$
      &
       \Leq
      &
       \Equ
      &
       \Neq
      &
       \Low
      &
       \Grt
      &
       \Geq
      \\
      \cline{1-7}
       $\comp{\prec}$
      &
       \Grt
      &
       \Neq
      &
       \Equ
      &
       \Geq
      &
       \Leq
      &
       \Low
      \\
\cline{1-7}
      \end{tabular}
    \end{center}
    \medskip
\end{itemize}
Using the theory atoms of theory \LC,
we can write the running example from the introduction as the \tprogram:
\comment{T: replaced $\wedge$ by ``,'', and $\geq$ by ``$>=$''
J: As a T-program the syntax seems correct, doesn't it?}
\begin{align}
  \bot & \leftarrow  \neg \mathtt{a},\, \Sum{s}\!>=\!120 \label{f:ex2.1}\\
  \mathtt{a} & \leftarrow  \Sum{s}\!>\!100 \label{f:ex2.2}
\end{align}
 \subsection{A new and simpler definition of theory stable model}\label{sec:simpler}

In this section,
we introduce a new, simplified version of the definition of $\tuple{\AT,\TheoryAtomsES}$\nobreakdash-\emph{stable model} that,
under certain conditions introduced below, is equivalent to Definition~\ref{def:stable.model}.
\begin{definition}[Theory stable model simplified]\label{def:stable.model.simple}
    Given a theory~$\AT=\langle \TheoryAtoms, \Sat, \comp{\cdot} \,\rangle$ and
    a set~${\TheoryAtomsES \subseteq \TheoryAtoms}$ of external theory atoms,
    a set $X\subseteq \Atoms\cup\TheoryAtoms$ of atoms is a $\langle\AT,\TheoryAtomsES\rangle$\nobreakdash-\emph{stable model} of a \tprogram\ $P$,
    if $(X \cap \TheoryAtoms) \in \Sat$ and $X$ is a regular stable model of the theory
    \begin{align}
        P
        \cup
        \{ \ {\thAt \vee \comp{\thAt} } \, \mid \ \thAt\in \TheoryAtomsES \ \}
        \label{eq:transformationB.simple}
        \ .
    \end{align}
\end{definition}
This definition drops the existential quantifier used in Definition~\ref{def:stable.model} for identifying a~\tsolution~$S$.

To state the conditions under which Definitions~\ref{def:stable.model} and~\ref{def:stable.model.simple} are equivalent,
we need the following concepts~\citep{cafascwa23a}.
An abstract theory $\AT = \langle \TheoryAtoms, \Sat, \comp{\cdot} \,\rangle$
is \emph{consistent} if none of its satisfiable sets contains complementary theory atoms, that is, there is no $S \in \Sat$ such that $\thAt\in S$ and $\comp{\thAt}\in S$ for some atom~$\thAt\in\TheoryAtoms$.
A set~$S$ of theory atoms is \emph{closed} if~$\thAt \in S$ implies~$\comp{\thAt} \in S$.
A set~$S$ of theory atoms is called $\TheoryAtomsES$\nobreakdash-\emph{complete},
if for all $\thAt \in \TheoryAtomsES$, either $\thAt\in S$ or $\comp{\thAt}\in S$.
A theory~$\AT=\langle \TheoryAtoms, \Sat, \comp{\cdot} \,\rangle$  is \emph{monotonic}\comment{T: ?}
if~$S \subseteq S'$ and~$S' \in \Sat$ implies~$S \in \Sat$.
Programs over a consistent theory with a closed set of external theory atoms have the following interesting properties:
\begin{proposition}[\mbox{\citealp[Proposition~2]{cafascwa23a}}]\label{prop:consistentclosed}
  For a consistent abstract theory~$\AT = \langle \TheoryAtoms, \Sat, \comp{\cdot} \,\rangle$ and
  a closed set $\TheoryAtomsES\subseteq \TheoryAtoms$ of external theory atoms,
all \tsolution s $S\subseteq\TheoryAtoms$ are \TheoryAtomsES-complete.
\end{proposition}

\begin{theorem}\label{thm:stable.model.simple}
  Given a theory~$\AT=\langle \TheoryAtoms, \Sat, \comp{\cdot} \,\rangle$ and
  a set~${\TheoryAtomsES \subseteq \TheoryAtoms}$ of external theory atoms,
  if $\AT$ is consistent and monotonic, and $\TheoryAtomsES$ is closed,
  then Definitions~\ref{def:stable.model} and~\ref{def:stable.model.simple} are equivalent.
\end{theorem}

The preconditions of Theorem~\ref{thm:stable.model.simple} cover many hybrid extensions of \clingo\
such as \clingcon, \clingodl, and \clingolp.
Therefore, in the rest of the paper, we assume these conditions and use Definition~\ref{def:stable.model.simple} as the definition of a $\langle\AT,\TheoryAtomsES\rangle$\nobreakdash-stable model.
 \subsection{Structured and Compositional Theories}\label{sec:background.compositional.theories}

The approach presented in Sections~\ref{sec:clingo5} and~\ref{sec:simpler} is intentionally generic in its formal definitions.
No assumption is made on the syntax or inner structure of theory atoms.
An abstract theory is only required to
specify when a set of theory atoms is satisfiable and
provide a complement function for each theory atom.
Definition~\ref{def:stable.model.simple} is a bit more specific,
and further assumes consistent and monotonic theories (with a closed set of external atoms), something we may expect in most cases while still being very generic.
The advantage of this generality is that it allows us to accommodate external theories without requiring much knowledge about their behavior.
However, this generality comes at a price:
ignoring the structure of the external theory may prevent in depth formal elaborations,
such as, for instance, the study of strong equivalence for logic programs with constraint atoms.

Also, in many practical applications of hybrid systems, we are interested in the assignment of values to variables
rather than the theory atoms that are satisfied,
something not reflected in Definitions~\ref{def:stable.model} or~\ref{def:stable.model.simple}.
This is what happens, for instance, in most hybrid extensions of answer set solvers
including the hybrid versions of \clingo, namely, \clingcon, \clingoM{dl}, and \clingoM{lp}.
Since the presence of variables in theory atoms can be exploited to describe their semantics in more detail,
we resort to refined types of theories that are enriched with a specific structure.
\begin{definition}[Structured theory; \citealp{cafascwa23a}]\label{def:structure}
  Given an abstract theory~$\AT=\langle \TheoryAtoms, \Sat, \comp{\cdot} \,\rangle$,
  we define its \emph{structure} as a tuple
  \(
  (\XAT,\DAT, \varsATfunction_\AT, \denAT{\cdot}),
  \)
  where
  \begin{enumerate}
  \item $\XAT$ is a set of variables,
  \item $\DAT$ is a set of domain elements,
  \item ${\varsATfunction_\AT: \TheoryAtoms \rightarrow 2^{\XAT}}$ is a function returning the set of variables contained in a theory atom
    such that~$\VARSAT{\thAt} = \VARSAT{\comp{\thAt}}$ for all theory atoms~$\thAt\in\TheoryAtoms$,
  \item $\VAT=\{v \mid v: \XAT \rightarrow \DAT\}$ is the set of all \emph{assignments} over~$\XAT$ and~$\DAT$, and
  \item\label{def:structure.5}
    $\denAT{\cdot}: \TheoryAtoms \rightarrow 2^{\VAT}$ is a function mapping theory atoms to sets of valuations such that
    \[
      v \in \denAT{\thAt} \text{ iff } w \in \denAT{\thAt}
    \]
    for all theory atoms~$\thAt\in\TheoryAtoms$ and every pair of valuations~$v,w$ agreeing on the value of all variables~$\VARSAT{\thAt}$ occurring in~$\thAt$.
  \end{enumerate}
\end{definition}
Whenever an abstract theory \AT\ is associated with such a structure, we call it \emph{structured}
(rather than abstract).
Observe that structured theories are based on the same concepts as \HTC, viz.\
a domain, a set of variables, valuation functions, and denotations for atoms.
In fact, we assume that the previously introduced definitions and notation for these concepts still apply here.
This similarity is intended, since \HTC\ was originally thought of as a generalization of logic programs with theory atoms.
One important difference between an \HTC\ valuation and a structured theory valuation $v \in \VAT$ is that
the latter cannot leave a variable undefined in \HTC\ terms, that is,
for a structured theory, we assume that the ``undefined'' value is not included in the domain $\undefined \not\in \DAT$,
and so, $v(x) \neq \undefined$ for all $x \in \XAT$ and $v \in \VAT$.
We also emphasize the use of the theory name $\AT$ subindex in all the components of a structure because, as we see below,
an \HTC\ characterization allows us to accommodate multiple theories in the same formalization.

Given a set $S$ of theory atoms, we define its denotation as $\denAT{S} \eqdef \bigcap_{\thAt\in S}\denAT{\thAt}$.
For any structured theory~$\AT = \langle \TheoryAtoms, \Sat, \comp{\cdot} \,\rangle$ with structure $(\XAT,\DAT, \varsATfunction_\AT, \denAT{\cdot})$,
if~$X \subseteq \Atoms \cup \TheoryAtoms$ is a set of atoms, we define~$\Answer{X}$ as the set
\[
\{ (Y,\restr{v}{\Sigma}) \mid v \in \denAT{X \cap \TheoryAtoms} \}
\quad
\text{ with }
Y = X \cap \Atoms
\text{ and }
\Sigma = \VARSAT{X\cap\TheoryAtoms}
\ .
\]
Intuitively,
$\Answer{X}$ collects all pairs $(Y,v')$ where $Y$ is fixed to the regular atoms in $X$ and
$v'$ varies among all valuations in $\denAT{X \cap \TheoryAtoms}$
restricted to the variables occurring in the theory atoms in~$X$.
Note that $v$ is only defined for a subset of variables, namely, $\Sigma=\VARSAT{X\cap\TheoryAtoms}\subseteq \X$.
\begin{definition}[Answer set; \citealp{cafascwa23a}]
  If~$X$ is a $\langle\AT,\TheoryAtomsES\rangle$-stable model of program~$P$ and~$(Y,v)\in\Answer{X}$,
  then $(Y,v)$ is a $\langle\AT,\TheoryAtomsES\rangle$\nobreakdash-\emph{answer set} of~$P$.
\end{definition}

Let $\AT = \tuple{\TheoryAtoms, \Sat, \comp{\cdot}}$ be a theory with structure $(\XAT,\DAT, \varsATfunction_\AT, \denAT{\cdot})$.
We say that $\AT$ has an \emph{absolute complement} whenever
the denotation of~$\comp{\thAt}$ is precisely the set complement of~$\denAT{\thAt}$, that is,
$\denAT{\comp{\thAt}} = \VAT \setminus \denAT{\thAt}$.
We also say that $\AT$ is \emph{compositional},
when it satisfies
${\Sat = \{S \subseteq \TheoryAtoms \mid \denAT{S} \neq \emptyset\}}$, that is, a set~$S$ is $\AT$-satisfiable iff its denotation is not empty.
This means that, for compositional theories, the set~$\Sat$ of~$\AT$\nobreakdash-satisfiable sets does not need to be explicitly stated as it can be derived from the denotation.
Hence,
we can write~$\AT=\langle \TheoryAtoms, \comp{\cdot} \,, \XAT,\DAT, \varsATfunction_\AT, \denAT{\cdot}\rangle$ to denote a compositional theory.

As an example for a compositional, structured theory with an absolute complement,
let us associate the theory of linear equations \LC\ with the structure $(\X_\LC,\D_\LC,\varsTfunction{\LC},\den{\cdot}_\LC)$, where
\begin{itemize}
\item $\X_\LC$ is an infinite set of integer variables,
\item $\D_\LC=\mathbb{Z}$,
\item ${\varsT{\LC}{\Sum{k_1\!*\!x_1;\!{\small\dots}\!;k_n\!*\!x_n} \prec k_0}=\{x_1,\!{\small\dots}\!,x_n\}}$, and
\item
  \(
  \den{\Sum{k_1*x_1;\dots;k_n*x_n}  \prec k_0}_\LC
  =
  \)
  \[\textstyle
    \big\{
    v\in\mathcal{V}_\LC \mid \{k_1,v(x_1),\dots k_n,v(x_n)\}\subseteq\mathbb{Z}
    \text{ and }
    \sum_{1\leq i \leq n}k_i*v(x_i)\prec k_0
    \big\}
    \ .
  \]
\end{itemize}
With \LC, a set $S$ of theory atoms capturing linear equations is \LC-satisfiable whenever
\(
\den{S}_\LC
\)
is non-empty.

In general,
we have a one-to-many correspondence between $\langle\AT,\TheoryAtomsES\rangle$-stable models and
their associated $\langle\AT,\TheoryAtomsES\rangle$\nobreakdash-answer sets.
That is, $\Answer{X}$ is generally no singleton for a stable model $X$.
However, if we focus on consistent, compositional theories that have an absolute complement, then we can establish a one\nobreakdash-to\nobreakdash-one correspondence between any~$\langle\AT,\TheoryAtomsES\rangle$\nobreakdash-stable model of a program and a kind of equivalence classes among its associated $\langle\AT,\TheoryAtomsES\rangle$\nobreakdash-answer sets.
Each of the equivalence classes may contain many $\langle\AT,\TheoryAtomsES\rangle$\nobreakdash-answer sets, but any of them has enough information to reconstruct the corresponding~$\langle\AT,\TheoryAtomsES\rangle$\nobreakdash-stable model.
An answer set $(Y,v)$ \emph{satisfies} an atom $b \in \Atoms \cup \TheoryAtoms$, written $(Y,v) \models b$, if
\begin{itemize}
  \item $b\in Y$ \quad\ for regular atoms $b \in \Atoms$ and
  \item $v\IN\denAT{b}$ \ for theory atoms $b\in\TheoryAtoms$.
\end{itemize}
where~$v \IN \denAT{b}$ holds iff there exists some valuation ${w \in \denAT{b}}$ such that $v$ and $w$ agree on the variables in $\dom(v)$.
Note that we use~$v\IN \denAT{b}$ instead of~$v \in \denAT{b}$ because $v$ can be partial and, if so,
we just require that there exists some complete valuation in $\denAT{b}$ that
agrees with the values assigned by $v$ to its defined variables $\dom(v)$.
For any negative literal $\neg b$, we say that $(Y,v) \models \neg b$ simply when $(Y,v) \not\models b$.
If $B$ is a rule body, we write $(Y,v) \models B$ to stand for $(Y,v) \models \ell$ for every literal $\ell$ in $B$.
For a program~$P$ and an $\langle\AT,\TheoryAtomsES\rangle$\nobreakdash-\emph{answer set}~$(Y,v)$, we define
\begin{align*}
  \Stable{Y,v} & = Y
                 \cup  \big\{ \thAt \in \TheoryAtomsES \mid  v \IN \denAT{\thAt} \big\}
                 \cup  \big\{ \thAt \in \TheoryAtomsDN \mid (\thAt \leftarrow B) \in P, (Y,v) \models B \big\}
\end{align*}
We also write~$(Y_1,v_1) \sim (Y_2,v_2)$ if~$\Stable{Y_1,v_1} = \Stable{Y_2,v_2}$
and say that~$(Y_1,v_1)$ and~$(Y_2,v_2)$ belong to the same equivalence class with respect to~$P$.
\begin{proposition}[\mbox{\citealp[Proposition~8]{cafascwa23a}}]
    \label{prop:answer.sets.correspondence}
Let ~$\AT=\langle \TheoryAtoms, \comp{\cdot} \,, \XAT,\DAT, \varsATfunction_\AT, \denAT{\cdot}\rangle$ be a compositional theory with an absolute complement and let $\TheoryAtomsES\subseteq \TheoryAtoms$ be a closed set of external atoms.
There is a one\nobreakdash-to\nobreakdash-one correspondence between
the $\langle\AT,\TheoryAtomsES\rangle$\nobreakdash-stable models of a program~$P$ and
the equivalence classes with respect to~$P$ of its $\langle\AT,\TheoryAtomsES\rangle$\nobreakdash-answer sets.

Furthermore,
if~$(Y,v)$ is a~$\langle\AT,\TheoryAtomsES\rangle$\nobreakdash-answer set,
then~$\Stable{Y,v}$ is a $\langle\AT,\TheoryAtomsES\rangle$\nobreakdash-stable model of~$P$ and~$(Y,v)$ belongs to~$\Answer{\Stable{Y,v}}$.
\end{proposition}

In the following,
we restrict ourselves to consistent compositional theories with an absolute complement and refer to them just as theories.

 \subsection{\HTC-characterization}
\label{sec:htcsemantics}

We now present a direct encoding of \tprogram{s} as \HTC\ theories.
This encoding is ``direct'' in the sense that it preserves the structure of each program rule by rule and atom by atom,
only requiring the addition of a fixed set of axioms.
As a first step, we start embodying compositional theories in \HTC\ by mapping their respective structures
(domain, variables, valuations, and denotations)
while having in mind that \HTC\ may tolerate multiple abstract theories in the same formalization.
For this reason, when we encode a theory
$\AT= \langle \TheoryAtoms_\AT, \comp{\cdot}\, , \X_\AT,\D_\AT,\varsATfunction_\AT, \den{\cdot}_\AT \rangle$ into an
\HTC\ theory over a signature $\tuple{\X,\D,\C}$,
we only require $\X_\AT \subseteq \X$ and $\D_\AT \subseteq \D$,
so that the \HTC\ signature may also include variables and domain values other than the ones in \AT.
We then map each abstract theory atom $\thAt \in \TheoryAtoms$ into a corresponding \HTC\nobreakdash-constraint atom ${\htcsemanticsA(\thAt) \in \C}$
with the same variables~${\vars{\htcsemanticsA(\thAt)} = \varsT{\AT}{\thAt}}$.
We also require the following relation between the theory denotation of a theory atom and
the \HTC\ denotation of its corresponding constraint atom:
\begin{align}
\den{\htcsemanticsA(\thAt)} & \ \eqdef \
     \{v\in\Vu^{\X,\D}  \mid  \exists w\in\den{\thAt}_\AT, \ \restr{v}{\XAT}=w \} \label{f:atomden}
\end{align}
Note that \HTC\ valuations $v\in\V^{\X,\D}$ apply to a (possibly) larger set of variables ${\X \supseteq \X_\AT}$
and on larger domains ${\Du \supset \D_\AT}$,
which include the element $\undefined \not\in \D_\AT$ to represent undefined variables in \HTC.
The denotation $\den{\htcsemanticsA(\thAt)}$ collects all possible \HTC-valuations that coincide with some \AT-valuation $w \in \den{\thAt}_\AT$ on the variables of theory~$\AT$, letting everything else vary freely.

We write $\htcsemanticsA(S)$ for $\{ \htcsemanticsA(\thAt) \mid \thAt \in S \}$ and~$S \subseteq \TheoryAtoms$.
The following result shows that this mapping of denotations preserves \AT-satisfiability:
\begin{proposition}\label{prop:translation.satisfiable}
  Given a theory~$\AT= \langle \TheoryAtoms_\AT, \comp{\cdot}\, , \X_\AT,\D_\AT,\varsATfunction_\AT, \den{\cdot}_\AT \rangle$, a set~${S \subseteq \TheoryAtoms}$ of theory atoms is \AT\nobreakdash-satisfiable iff~$\htcsemanticsA(S)$ is satisfiable in~$\HTC$.
\end{proposition}

Let us now consider \tprogram{s} $P$ over~$\langle\Atoms,\TheoryAtoms\rangle$,
that are associated with a structured theory~${\AT= \langle \TheoryAtoms_\AT, \comp{\cdot}\, , \X_\AT,\D_\AT,\varsATfunction_\AT, \den{\cdot}_\AT \rangle}$.
As explained in Section~\ref{sec:htc},
we can capture regular programs by identifying regular atoms~$\Atoms$ with both variables and constraint atoms and
having the truth constant~$\true$ in the domain.
We can capture this by adding regular atoms both as variables and as constraint atoms to the signature of the \HTC\ theory, so we have
\begin{align}
  \Atoms \cup \X_\AT &\subseteq \X
  \label{eq:ex.htc.variables}
  \\
  \Atoms \cup \{ \tau(\thAt) \mid \thAt \in \TheoryAtoms_\AT \} &\subseteq \C
  \label{eq:ex.htc.atoms}
\end{align}
and the truth constant~$\true$ in the domain, so we have~~$\D_\AT \cup \{ \mathbf{t} \} \subseteq \D$.
For each regular atom~$\regAt \in \Atoms$, its variables and denotation are defined as follows:
\begin{itemize}
  \item $\vars{\regAt}=\{\vregAt\}$, and
  \item $\den{\regAt} = \den{\Dom{\{\mathbf{t}\}}{\vregAt}} = \{v \in \V^{\X,\D} \mid v(\vregAt)=\mathbf{t}\}$.
\end{itemize}
With this encoding, each regular atom $\regAt \in \Atoms$ is mapped into itself in the \HTC\nobreakdash-theory,
viz.\ $\htcsemanticsA(\regAt) \eqdef \regAt$.

Since the resulting~$\HTC$-theory contains variables from the abstract theory and the regular atoms,
we are interested in forcing variables to range only over their corresponding subdomain,
one from the external theory and one for Boolean values.
This is achieved by including an axiom of form
\begin{eqnarray}
\df(x) \to \Dom{\D_\AT}{x} \label{f:domain}
\end{eqnarray}
for each $x \in \X_\AT$ and each abstract theory $\AT$ encoded in our \HTC\ formalization.
The application of any \uvaluation\ $v \in \Vu^{\X_\AT,\D}$ satisfying~\eqref{f:domain}
to a variable $x \in \X_\AT$ from the abstract theory
returns some element from the theory domain $v(x) \in \D_\AT$ or is undefined, viz.\ $v(x) = \undefined$.
Let us denote by $\DOM{\AT,\Atoms}$ the set of all the axioms of the form of~\eqref{f:domain} for each $x \in \X_\AT$ plus all the axioms of the form \eqref{f:booldomain} for every ${\regAt \in \Atoms}$.

We are now ready to introduce the direct translation of a \tprogram\ $P$ over~$\tuple{\Atoms,\TheoryAtoms}$ with external theory~$\AT$ into an \htctheory~$\htcsemanticsP(P,\AT,\TheoryAtomsES)$ where~$\TheoryAtomsES$ is a set of theory atoms considered external in~$P$.
Given any rule $r$ of the form of~\eqref{theory:rule}, we let $\htcsemanticsA^B(r)$ stand for the formula
\begin{align}\label{eq:htcsemantics:body}
  \htcsemanticsA(\anyAt_1)\land\dots\land \htcsemanticsA(\anyAt_n)\land\neg \htcsemanticsA(\anyAt_{n+1})\land\dots\land\neg \htcsemanticsA(\anyAt_{m})
\end{align}
representing the body of $r$, where each occurrence of an atom $b_i$ in the program is replaced by $\htcsemanticsA(b_i)$.
We extend the application of $\htcsemanticsA$ to the whole rule $r$ and let $\htcsemanticsA(r)$ stand for
\begin{align}\label{eq:htcsemantics:rule}
  \htcsemanticsA^B(r) \rightarrow \htcsemanticsA(\anyAt_0)
\end{align}
assuming that, when the head is $\anyAt_0=\bot$, its translation is simply $\htcsemanticsA(\bot) \eqdef \bot$.
We further write ${\htcsemanticsA(P) \eqdef \{\htcsemanticsA(r) \mid r \in P \}}$, applying our transformation to all rules in program $P$.
The complete translation of \tprogram~$P$ is defined as
\begin{align}\label{eq:htcsemantics:complete}
  \htcsemanticsP(P,\AT,\TheoryAtomsES) \eqdef \htcsemanticsA(P) \cup \SEM{\TheoryAtomsES} \cup \DOM{\AT,\Atoms}
\end{align}
where $\SEM{\TheoryAtomsES}$ consists of
\begin{align}\label{eq:htcsemantics:external}
\df(x)
  \qquad
  \text{ for every } x \in \vars{\thAt} \text{ and }\thAt \in \TheoryAtomsES.
\end{align}
Formula~\eqref{eq:htcsemantics:external} asserts that each variable in any external atom can be arbitrarily assigned some (defined) value.
\begin{proposition}
  Let $P$ be a \tprogram\ over $\tuple{\Atoms,\TheoryAtoms}$ with external theory~$\AT$ and external atoms~$\TheoryAtomsES$.
For every~$\thAt \in \TheoryAtomsES$,
  \begin{gather*}
    \htcsemanticsP(P,\AT,\TheoryAtomsES) \models \htcsemanticsA(\thAt) \vee \htcsemanticsA(\comp{\thAt}).
  \end{gather*}
\end{proposition}
This proposition means that either $\thAt$ or its complement~$\comp{\thAt}$ must hold.
In other words,
axioms of form~\eqref{eq:htcsemantics:external} entail a kind of \emph{strong excluded middle} for external atoms
similarly to the simplified definition of a $\langle\AT,\TheoryAtomsES\rangle$-stable model
(in Definition~\ref{def:stable.model.simple}).
This is a stronger version of the usual choice construct~${\htcsemanticsA(\thAt)\vee\neg\htcsemanticsA(\thAt)}$ in \HT:
we can freely add $\htcsemanticsA(\thAt)$ or not but, when the latter happens,
we further provide evidence for the complement~$\htcsemanticsA(\comp{\thAt})$.

Following our running example,
consider program~$P$ consisting of rules~\eqref{f:ex2.1} and~\eqref{f:ex2.2}, plus the fact
\begin{align}\label{eq:ex.htc.fact}
  \Sum{s} = 130.
\end{align}
This program has a unique answer set where~$\regAt$ is true and~$s$ is assigned the value~$130$.

We assume the following translation of constraint atoms:
\begin{align*}
  \htcsemanticsA(\Sum{s} >=120) \ &\eqdef \ s\geq120
  \\
  \htcsemanticsA(\Sum{s} > 100) \ &\eqdef \ s>100
  \\
  \htcsemanticsA(\Sum{s} = 130) \ &\eqdef \ s=130
\end{align*}
Hence,  $\htcsemanticsP(P)$ is the \HTC\ theory
$\{\eqref{eq:ex.htc.contraint},\,\eqref{eq:ex.htc.rule},\, {s=130}\}$.
Given that constraint atoms ${\Sum{s}\geq120}$ and~${\Sum{s}>100}$ occur in the body, they must be external.
We assume that constraint atom~\eqref{eq:ex.htc.fact} is not external (see discussion below).
Thus, $\htcsemanticsP(P,\AT,\TheoryAtomsES)$ is the result of adding to~$\htcsemanticsP(P)$ the following axioms:
\begin{align}
  &\df(\vregAt) \rightarrow \regAt \tag{\ref{f:booldomain}}
  \\
  &\df(s) \rightarrow \Dom{\D_\LC}{s}
  \label{eq:ex.htc.domain.integer}
  \\
  &\df(s)
  \label{eq:ex.htc.domain.fact}
\end{align}
We can replace the last two axioms simply by~$\Dom{\D_\LC}{s}$,
which ensures that the variable~$s$ is assigned a value from the domain~$\D_\LC$ of linear constraints, namely, an integer.
Recall that~\eqref{f:booldomain} ensures that~$\vregAt$ is assigned a Boolean value, either~$\true$ or~$\undefined$.
This \HTC\nobreakdash-theory has a unique stable model
$$\{ \vregAt \mapsto \true, \, s \mapsto 130 \}.$$
This stable model corresponds to the unique answer set of program~$P$.

The following theorem states the relation between the answer sets of a \tprogram\ $P$ and the equilibrium models of its translation $\htcsemanticsP(P,\AT,\TheoryAtomsES)$ which we saw in the previous example holds in general.
\begin{theorem}\label{thm:first_translation}
  Let $P$ be a \tprogram\ over $\langle\Atoms,\TheoryAtoms\rangle$ with \TheoryAtomsES\ being a closed
  set of external atoms, and
  let ${\AT=\langle \TheoryAtoms, \mathcal{S}, \comp{\cdot} \,\rangle}$ be a consistent, compositional theory.
Then, there is a one-to-one correspondence between the $\langle\AT,\TheoryAtomsES\rangle$\nobreakdash-answer sets of~$P$ and
  the equilibrium models of~$\htcsemanticsP(P,\AT,\TheoryAtomsES)$ such that
  $(Y,v)$ is a $\langle\AT,\TheoryAtomsES\rangle$\nobreakdash-answer set of $P$ iff~$\tuple{t,t}$ is an equilibrium model of theory~$\htcsemanticsP(P,\AT,\TheoryAtomsES)$
  with~$t = v \cup \{ p \mapsto \true \mid p \in Y \}$.
\end{theorem}

This result shows that the semantics of a \tprogram~$P$ can alternatively be described as the equilibrium models of the \htctheory~$\htcsemanticsP(P,\AT,\TheoryAtomsES)$.
Intuitively, formulas of form~$\eqref{eq:htcsemantics:rule}$ capture the rules in the \tprogram\ and are used for the same purpose,
that is, to decide which founded atoms from $\TheoryAtoms \setminus \TheoryAtomsES$ can be eventually derived.
Furthermore, due to the minimization imposed to obtain an equilibrium model $\tuple{t,t}$,
if a founded atom $\thAt$ is not derived (that is, $\tuple{t,t} \not\models \htcsemanticsA(\thAt)$),
then all its variables $x \in \varsT{\AT}{\thAt}$ not occurring in external atoms or other derived atoms are left undefined,
viz.\ $t(x)=\undefined$.

An interesting consequence of the characterization provided by Theorem~\ref{thm:first_translation} is that all atoms whose variables occur in external atoms are implicitly external, even if we do not declare them as that.
This result is trivial when looked at the definition of~$\htcsemanticsP(P,\AT,\TheoryAtomsES)$, but it is far from obvious when considering the definition of $\langle\AT,\TheoryAtomsES\rangle$\nobreakdash-answer sets based on Definitions~\ref{def:stable.model} or~\ref{def:stable.model.simple}.
\begin{corollary}
  Let $P$ be \tprogram{s} over
  $\langle\Atoms,\TheoryAtoms\rangle$
  where
  $\TheoryAtomsES$ and~$\TheoryAtomsES'$ are sets of external theory atoms such that $\TheoryAtomsES \subseteq \TheoryAtomsES'$,
  and
  let \AT\ be a consistent, compositional theory.
If every $\thAt \in \TheoryAtomsES'$ satisfies~$\varsAT{\thAt} \subseteq \varsAT{\TheoryAtomsES}$,
  then
  the~$\langle\AT,\TheoryAtomsES\rangle$\nobreakdash-answer sets of $P$ and
  the $\langle\AT,\TheoryAtomsES'\rangle$-answer sets of~$P$ coincide.
\end{corollary}

In other words, extending the set of external atoms by adding new ones whose variables are already occurring in other external atoms does not change the answer sets of the program.
For instance, continuing with our running example, we have assumed that constraint atom~\eqref{eq:ex.htc.fact} is not external, but because its only variable~$s$ occurs in external atoms, it is implicitly external.
As stated above, assuming that it is external instead, does not change the answer sets of the program.

\section{Strong Equivalence of \tprogram{s}}\label{sec:stronge.lp}

Finally,
it is time to return our attention to the question we pose in the introduction:
can we remove rule~\eqref{f:ex2.1} from any program containing rules~\eqref{f:ex2.1} and~\eqref{f:ex2.2} without changing its semantics?
Continuing with our running example,
recall that
$$
\htcsemanticsP(P,\AT,\TheoryAtomsES) \ = \ \{\eqref{f:booldomain},  \eqref{eq:ex.htc.contraint}, \eqref{eq:ex.htc.rule}, \eqref{eq:ex.htc.domain.integer}, \eqref{eq:ex.htc.domain.fact} \}
$$
and, that we saw in Section~\ref{sec:strongeq} that~$\{\eqref{eq:ex.htc.contraint}, \eqref{eq:ex.htc.rule} \}$ and~$\eqref{eq:ex.htc.rule}$ are strongly equivalent.
Therefore, in any theory containing~\eqref{eq:ex.htc.rule}, we can remove~\eqref{eq:ex.htc.contraint} without changing its equilibrium models.
By Theorem~\ref{thm:first_translation}, this means that, any program~$P'$ satisfying
$$
\htcsemanticsP(P',\AT,\TheoryAtomsES) = \{\eqref{f:booldomain},  \eqref{eq:ex.htc.rule}, \eqref{eq:ex.htc.domain.integer}, \eqref{eq:ex.htc.domain.fact} \}
$$
has the same answer sets as program~$P$.
That is, in this particular example, we can safely remove rule~\eqref{f:ex2.1} without changing the semantics of the program.
This same reasoning applies to any program containing these two rules and not only to this one.

We can formalize this idea by defining strong equivalence of \tprogram{s}.
\begin{definition}[Strong equivalence of \tprogram{s}]
  \tprogram{s} $P$ and $Q$ are \emph{strongly equivalent} with respect to an external theory~$\AT$ and
  a set~$\TheoryAtomsES$ of theory atoms considered external,
  whenever \HTC-theories $\htcsemanticsP(P,\AT,\TheoryAtomsES)$ and $\htcsemanticsP(Q,\AT,\TheoryAtomsES)$ are strongly equivalent.
\end{definition}

\begin{corollary}
  If \tprogram{s}~$P$ and~$Q$ are strongly equivalent, then~$P \cup R$ and~$Q \cup R$ have the same answer sets for any program~$R$.
\end{corollary}

Our main result is an immediate consequence of Theorems~\ref{th:strongeq-suf} and~\ref{th:strongeq-nec}.
\begin{theorem}[Main Result]
  Two \tprogram{s} $P$ and $Q$ are strongly equivalent with respect to an external theory~$\AT$ and a set~$\TheoryAtomsES$ of theory atoms considered external, if and only if they are equivalent in \HTC.
\end{theorem}
This result provides a method to establish whether two logic programs with constraints are strongly equivalent.
In particular, applied to the example in the introduction, it shows that we can remove the first rule without changing the semantics of the program in any context that contains the second one.

 \section{Complexity}\label{sec:complexity}

In this section, we address the computational complexity of the satisfiability and strong equivalence problems for \tprogram s.
Note that the complexity of~\tprogram{s} highly depends on the associated external theory~$\AT$, and
is in general undecidable as illustrated by the following example.
\begin{example}\label{ex:diophantine}\nobreak Let ${\DIE= \tuple{\TheoryAtoms, \comp{\cdot}, \XAT,\DAT, \varsATfunction_\AT, \denAT{\cdot}}}$ be a structured theory
    with variables ${\XAT = \{ x_1, x_2, \dotsc \}}$, whose domain is the integers and whose atoms are of the form
    \begin{align*}
        \Sum{c_1 \cdot x_1^{e_1},\dots,c_n \cdot x_n^{e_n}} = 0
    \end{align*}
    where each~$c_i\in\mathbb{Z}$, $e_i\in\mathbb{N}$ and~$x_i$ is a variable.
The denotation~$\denAT{\cdot}$ of these atoms is the set of all assignments of integer values to the variables that satisfy the corresponding Diophantine equation.
\end{example}

Since the satisfiability problem for Diophantine equations is undecidable~\citep{matiiasevich93a},
the satisfiability problem for \tprogram{s} with theory~$\DIE$ is also undecidable.
The \emph{satisfiability problem} for \tprogram s\ consists of deciding whether there is a set of atoms that is a~$\tuple{\AT,\TheoryAtomsES}$\nobreakdash-stable model of some program.
\begin{theorem}\label{thm:complexity.tprogram.undecidable}
  The satisfiability problem for \tprogram{s} with theory~$\DIE$ is undecidable,
  even if the program is restricted to be a single fact.
\end{theorem}

We can provide a more detailed analysis of the complexity in terms of oracles that can solve the satisfiability problem for the external theory~$\AT$.
In what follows, we assume that the reader is familiar with the basic concepts of complexity theory.\footnote{Cf., e.g.~\cite{papadimitriou94a} for a comprehensive treatise on this subject.}
For convenience, we briefly recapitulate the definitions and some
elementary properties of the complexity classes considered in our analysis.
The class~$\mathsf{NP}$ consists of all decision problems that can be solved by a non\nobreakdash-deterministic Turing machine in polynomial time.
As usual, for any complexity class $\mathsf{C}$, by ${\mathsf{co\text{-}C}}$ we understand the class of all problems which are complementary to the problems in~$\mathsf{C}$.
Thus, $\mathsf{co\text{-}NP}$ is the class of all problems whose complements are in~$\mathsf{NP}$.
Furthermore, for complexity classes~$\mathsf{C}$ and~$\mathsf{O}$,
we denote by~$\mathsf{C}^{\mathsf{O}}$ the class of problems which can be decided by Turing machines of the same sort and time bound as~$\mathsf{C}$, but with access to an oracle for the problems in~$\mathsf{O}$ that can solve problems of this complexity class in a single step.
\begin{theorem}\label{thm:complexity.tprogram.oracle}
    The satisfiability problem for \tprogram s\ with a theory~$\AT$ is decidable in~$\mathsf{NP}^{\mathsf{O}}$ where~$\mathsf{O}$ is an oracle that solves the satisfiability problem for the external theory~$\AT$.
\end{theorem}

Since regular ASP programs are a particular case of \tprogram{s},
the following result immediately follows from the $\mathsf{NP}$\nobreakdash-completeness of the satisfiability problem
for programs without disjunctions in ASP~\citep{daeigovo01a}.
\begin{corollary}
  Let~$\AT$ be a theory that is decidable in polynomial time.
  Then, the satisfiability problem for \tprogram s\ with theory~$\AT$ is~$\mathsf{NP}$\nobreakdash-complete.
\end{corollary}

Finally, we obtain the following results for the complexity of checking whether two \tprogram s\ are strongly equivalent.
\begin{theorem}\label{thm:complexity.tprogram.undecidable.se}\nobreak Deciding whether two \tprogram{s} are strongly equivalent is undecidable,
  even if both programs are restricted to be single facts.
\end{theorem}
\begin{theorem}\label{thm:complexity.tprogram.oracle.se}
  Deciding whether two \tprogram{s} with a theory~$\AT$ are strongly equivalent is decidable in~$\mathsf{coNP}^{\mathsf{O}}$,
  where~$\mathsf{O}$ is an oracle that solves the satisfiability problem for the external theory~$\AT$.
\end{theorem}

Since regular ASP programs are a particular case of \tprogram s, the following result immediately follows from the $\mathsf{co\text{-}NP}$\nobreakdash-completeness of the strongly equivalence problem for ASP problems~\citep{petowo09a}.
\begin{corollary}
  Let~$\AT$ be a theory that is decidable in polynomial time.
  Then, deciding whether two \tprogram s\ are strongly equivalent is~$\mathsf{coNP}$\nobreakdash-complete.
\end{corollary}

 \section{Conclusion}\label{sec:discussion}

We have introduced a method to establish whether two logic programs with constraints are strongly equivalent.
This method is based on a translation of logic programs with constraints into \HTC\nobreakdash-theories and the characterization of strong equivalence in the context of \HTC.
In particular, we have shown that two logic programs with constraints are strongly equivalent if and only if their translations are equivalent in \HTC\ and study the computational complexity of this problem.

The translation from logic programs with constraints into \HTC\nobreakdash-theories is of independent interest as it can serve as a tool to study strong equivalence of other extensions of logic program with constraints.
It is worth recalling that \HTC\ was introduced as a generalization of \HT\ to deal with constraint satisfaction problems
involving defaults over the variables of the external theory~\citep{cakaossc16a}.
Since then, \HTC\ has been extended to also formalize the semantics of aggregates over those same variables~\citep{cafascwa20a,cafascwa20b}.
As a future work, we are planning to leverage this translation to implement a new hybrid solver that can handle these two features (constraints with defaults and aggregates) and uses \clingo~\texttt{5} as a back-end solver.

\subsection*{Acknowledgments}

This work was supported by
DFG grant SCHA 550/15, Germany,
and by the National Science Foundation under Grant No. 95-3101-0060-402 and the CAREER award 2338635, USA.
Any opinions, findings, and conclusions or recommendations expressed in this material are those of the authors and do not necessarily reflect the views of the National Science Foundation.

\bibliographystyle{plainnat}

\section{Proofs}

\section{Proof of Theorem~\ref{thm:stable.model.simple}}

In all the following Lemmas we assume a consistent and monotonic theory~${\AT=\langle \TheoryAtoms, \Sat, \comp{\cdot} \,\rangle}$.

\begin{lemma}\label{lem:stable.model}
A set $X\subseteq \Atoms\cup\TheoryAtoms$ of atoms is a $\langle\AT,\TheoryAtomsES\rangle$\nobreakdash-\emph{stable model} of a \tprogram~$P$ according to Definition~\ref{def:stable.model} iff there is some $\TheoryAtomsES$\nobreakdash-complete set~$S\in \Sat$ such that $X$ is a regular stable model of the program~\eqref{eq:transformation}.
\end{lemma}

\begin{proof}
  $X$ is a $\langle\AT,\TheoryAtomsES\rangle$\nobreakdash-\emph{stable model} of a \tprogram\ $P$
  \\iff (by definition)
  \\there is some \tsolution\ $S$ such that $X$ is a regular stable model of the program~\eqref{eq:transformation}.
  \\iff (by Proposition~\ref{prop:consistentclosed})
  \\there is some \tsolution\ $S$ such that~$S$ is~$\TheoryAtomsES$\nobreakdash-complete and~$X$ is a regular stable model of the program~\eqref{eq:transformation} 
  \\iff (by definition)
  \\there is some $\TheoryAtomsES$\nobreakdash-complete set~$S\in \Sat$ such that~$\Comp_{\TheoryAtomsES}(S) \in \Sat$ and~$X$ is a regular stable model of the program~\eqref{eq:transformation}.
  \\iff
  \\there is some $\TheoryAtomsES$\nobreakdash-complete set~$S\in \Sat$ such that $X$ is a regular stable model of the program~\eqref{eq:transformation}.
  \\
  For the last equivalence note that~$Comp_{\TheoryAtomsES}(S) = S$ because~$S$ is $\TheoryAtomsES$\nobreakdash-complete.
\end{proof}

\begin{lemma}\label{lem:stable.model2}
    Let $S\in \Sat$ be a~$\TheoryAtomsES$\nobreakdash-complete and $X$ be a regular stable model of the program~\eqref{eq:transformation}.
Then, $X \cap \TheoryAtoms \subseteq S$.
\end{lemma}

\begin{proof}
    Pick~$\thAt \in X \cap \TheoryAtoms$.
Then, $\thAt \in \TheoryAtoms \cap H(P)$ or~$\thAt \in S \cap \TheoryAtomsES$.
The latter clarly implies that~$\thAt \in S$.
For the former, note that~$\thAt \in X$ implies that~$\bot \leftarrow \thAt$ does not bleong to~\eqref{eq:transformation}.
Then~$\thAt \in \TheoryAtoms \cap H(P)$ implies that~$\thAt \in S$.
\end{proof}

\begin{lemma}\label{lem:stable.model5}
    Let $S\in \Sat$ be a~$\TheoryAtomsES$\nobreakdash-complete set and $X$ be a regular stable model of the program~\eqref{eq:transformation}.
Then, $X \cap \TheoryAtomsES = S \cap \TheoryAtomsES$.
\end{lemma}

\begin{proof}
    $X \cap \TheoryAtomsES \subseteq S \cap \TheoryAtomsES$ follows by Lemma~\ref{lem:stable.model2}.
$X \cap \TheoryAtomsES \supseteq S \cap \TheoryAtomsES$ is a consequence of the second group of rules in~\eqref{eq:transformation}.
\end{proof}

\begin{lemma}\label{lem:stable.model3}
    Let $S\in \Sat$ be a~$\TheoryAtomsES$\nobreakdash-complete set and $X$ be a regular stable model of the program~\eqref{eq:transformation}.
Then, $X$ satisfies~$\bot \leftarrow \thAt,\, \comp{\thAt}$ for every~$\thAt\in \TheoryAtoms$.
\end{lemma}

\begin{proof}
    Suppose for the sake of contradiction that this is not the case, that is, $\{\thAt, \comp{\thAt}\} \subseteq X$.
By Lemma~\ref{lem:stable.model2}, this implies~$\{\thAt, \comp{\thAt}\} \subseteq S$, which is a contradiction with the fact that~$\AT$ is consistent.
\end{proof}

Consider program
\begin{align}\label{eq:transformation2}
    P
    \cup
    \{ \ {\thAt \vee \comp{\thAt} } \, \mid \ \thAt\in \TheoryAtomsES \ \}
    \cup
    \{ \ {\bot \leftarrow \thAt,\, \comp{\thAt}} \ \mid \ \thAt\in \TheoryAtoms \ \}
    \ .
\end{align}

\begin{lemma}\label{lem:stable.model4}
    Let $S\in \Sat$ be a~$\TheoryAtomsES$\nobreakdash-complete set and~$X \subseteq \Atoms \cup \TheoryAtomsES$ be a set of atoms such that~$X \cap \TheoryAtoms \subseteq S$.
Then, $X$ is a regular stable model of the program~${\eqref{eq:transformation}\cup  \{ \ {\bot \leftarrow \thAt,\, \comp{\thAt}} \ \mid \ \thAt\in \TheoryAtoms \ \}}$
    iff
    $X$ is a regular stable model of the program~${\eqref{eq:transformation2} \cup  \{ \ {\thAt \leftarrow } \, \mid \ \thAt\in (S\cap\TheoryAtomsES) \ \}}$.
\end{lemma}

\begin{proof}
\emph{Left-to-right}.
Every regular stable model of~$\eqref{eq:transformation}\cup  \{ \ {\bot \leftarrow \thAt,\, \comp{\thAt}} \ \mid \ \thAt\in \TheoryAtoms \ \}$
is a regular stable model of
\begin{align}\label{eq:transformation3}
    P
    \cup
    \{ \ {\thAt\leftarrow} \, \mid \ \thAt\in (S\cap\TheoryAtomsES) \ \}
    \cup
    \{ \ {\bot \leftarrow \thAt,\, \comp{\thAt}} \ \mid \ \thAt\in \TheoryAtoms \ \}
\end{align}
because the former is the result of adding some constraints to~\eqref{eq:transformation3}.
Furthermore, the stable models of~\eqref{eq:transformation3} and
\begin{gather}\label{eq:transformation3b}
    \eqref{eq:transformation2} \cup \{ \ {\thAt\leftarrow} \, \mid \ \thAt\in (S\cap\TheoryAtomsES) \ \}    
\end{gather}
conincide because~$S$ is~$\TheoryAtomsES$\nobreakdash-complete, which implies that a fact for one of the dijsunts belongs to~\eqref{eq:transformation3}.
\emph{Right-to-left}.
We show that every regular stable model of~\eqref{eq:transformation3} satisfies~${\bot \leftarrow \thAt}$ for every~$\thAt\in (\TheoryAtoms\cap\Head{P} \setminus S)$.
Pick an arbitray theory atom~$\thAt\in (\TheoryAtoms\cap\Head{P} \setminus S)$.
Since~$\thAt \notin S$, by assumption, it follows that~$\thAt \notin X$ and, thus, that~$X$ satisfies~$\bot \leftarrow \thAt$.
\end{proof}

\begin{lemma}\label{lem:stable.model6}
    $X$ is a regular stable model of~$\eqref{eq:transformation2}$ iff $X$ is a regular stable model of~$\eqref{eq:transformation2} \cup  \{ \ {\thAt \leftarrow } \, \mid \ \thAt\in (X\cap\TheoryAtomsES) \}$.
\end{lemma}

\begin{proof}
    Note that in the presence of the dijsuntion and constraint in~$\eqref{eq:transformation2}$, the regular stable models of~$\eqref{eq:transformation2} \cup  \{ \ {\thAt \leftarrow } \, \mid \ \thAt\in (X\cap\TheoryAtomsES) \}$
    and~$\eqref{eq:transformation2} \cup  \{ \ {\leftarrow \neg \thAt  } \, \mid \ \thAt\in (X\cap\TheoryAtomsES) \}$
    conincide.
Then, right-to-left direction is inmediate.
For the left-to-rigth is easy to see that~$X$ satisfies all the extra constraints.
\end{proof}

\begin{lemma}\label{lem:stable.model.simple.aux}
    A set $X\subseteq \Atoms\cup\TheoryAtoms$ of atoms is a $\langle\AT,\TheoryAtomsES\rangle$\nobreakdash-\emph{stable model} of a \tprogram\ $P$,
    iff $(X \cap \TheoryAtoms) \in \Sat$ and it is a regular stable model of the program~\eqref{eq:transformation2}.
\end{lemma}

\begin{proof}
    $X$ is a $\langle\AT,\TheoryAtomsES\rangle$\nobreakdash-\emph{stable model} of~$P$ according to Definition~\ref{def:stable.model} 
    \\iff (Lemma~\ref{lem:stable.model})
    \\there is some $\TheoryAtomsES$\nobreakdash-complete set~$S\in \Sat$ such that $X$ is a regular stable model of the program~\eqref{eq:transformation}
    \\iff (Lemma~\ref{lem:stable.model2} and Lemma~\ref{lem:stable.model5})
    \\there is some $\TheoryAtomsES$\nobreakdash-complete set~$S\in \Sat$ such that $X \cap \TheoryAtoms \subseteq S$ and $X \cap \TheoryAtomsES = S \cap \TheoryAtomsES$, and $X$ is a regular stable model of the program~\eqref{eq:transformation}
    \\iff (Lemma~\ref{lem:stable.model3})
    \\there is some $\TheoryAtomsES$\nobreakdash-complete set~$S\in \Sat$ such that $X \cap \TheoryAtoms \subseteq S$ and $X \cap \TheoryAtomsES = S \cap \TheoryAtomsES$, and $X$ is a regular stable model of the program~$\eqref{eq:transformation} \cup  \{ \ {\bot \leftarrow \thAt,\, \comp{\thAt}} \ \mid \ \thAt\in \TheoryAtoms \ \}$
    \\iff (Lemmas~\ref{lem:stable.model4})
    \\there is some $\TheoryAtomsES$\nobreakdash-complete set~$S\in \Sat$ such that $X \cap \TheoryAtoms \subseteq S$ and $X \cap \TheoryAtomsES = S \cap \TheoryAtomsES$, and $X$ is a regular stable model of the program~$\eqref{eq:transformation2} \cup  \{ \ {\thAt \leftarrow } \, \mid \ \thAt\in (S\cap\TheoryAtomsES) \ \}$
    \\iff (using the fact $X \cap \TheoryAtomsES = S \cap \TheoryAtomsES$)
    \\there is some $\TheoryAtomsES$\nobreakdash-complete set~$S\in \Sat$ such that $X \cap \TheoryAtoms \subseteq S$ and $X \cap \TheoryAtomsES = S \cap \TheoryAtomsES$, and $X$ is a regular stable model of the program~$\eqref{eq:transformation2} \cup  \{ \ {\thAt \leftarrow } \, \mid \ \thAt\in (X\cap\TheoryAtomsES) \ \}$
    \\iff (Lemma~\ref{lem:stable.model6})
    \\there is some $\TheoryAtomsES$\nobreakdash-complete set~$S\in \Sat$ such that $X \cap \TheoryAtoms \subseteq S$ and $X \cap \TheoryAtomsES = S \cap \TheoryAtomsES$, and $X$ is a regular stable model of the program~$\eqref{eq:transformation2}$
    \\iff
    \\$X \cap\TheoryAtoms \in \Sat$ and $X$ is a regular stable model of the program~$\eqref{eq:transformation2}$.
    \\[5pt]
    For the last equivalence, we proceed left-to-rigth and right-to-left.
\emph{Left-to-right.}
Since~$X \cap \TheoryAtoms \subseteq S$, $\AT$ is monotonic and~$S \in \Sat$, it follows that~$(X \cap \TheoryAtoms) \in \Sat$.
\emph{Right-to-left.}
Note that~$X$ is $\TheoryAtomsES$\nobreakdash-complete because of the dijsuntions in~$\eqref{eq:transformation2}$.
Then, just take~$S = X$.
\end{proof}

\begin{proof}[Proof of Theorem~\ref{thm:stable.model.simple}]
By Lemma~\ref{lem:stable.model.simple.aux},
it follows that $X$ is a $\langle\AT,\TheoryAtomsES\rangle$\nobreakdash-\emph{stable model} of~$P$ according to Definition~\ref{def:stable.model} iff $(X \cap \TheoryAtoms) \in \Sat$ and it is a regular stable model of the program~\eqref{eq:transformation2}.
Since~\eqref{eq:transformation2} is the result of adding some integrity constraints to~\eqref{eq:transformationB.simple}, it follows that, if~$X$ is a regular stable model of~\eqref{eq:transformation2}, then it is a regular stable model of~\eqref{eq:transformationB.simple}.
Hence, Definition~\ref{def:stable.model} implies Definition~\ref{def:stable.model.simple}.
Conversely, assume that~$X$ is a $\langle\AT,\TheoryAtomsES\rangle$\nobreakdash-\emph{stable model} of~$P$ according to Definition~\ref{def:stable.model.simple}.
Then, $(X \cap \TheoryAtoms) \in \Sat$ and it is a regular stable model of the program~\eqref{eq:transformationB.simple}.
Since $(X \cap \TheoryAtoms) \in \Sat$ and~$\AT$ is consistent, it follows that~$\{\thAt, \comp{\thAt} \} \not\subseteq X$ for every~$\thAt\in \TheoryAtoms$.
Hence, $X$ is a regular stable model of~\eqref{eq:transformation2} and a $\langle\AT,\TheoryAtomsES\rangle$\nobreakdash-\emph{stable model} of~$P$ according to Definition~\ref{def:stable.model}.
\end{proof}

 \begin{lemma}\label{lem:translation.valuations}
  Let~$S \subseteq \TheoryAtoms$ and~$t : \X \longrightarrow \Du$.
Then, $t \in \den{\htcsemanticsP(S)}$ iff~$\restr{t}{\XAT} \in \denAT{S}$.
\end{lemma}

\begin{proof}
  $t \in \den{\htcsemanticsP(S)}$
  \\iff
  $t \in \den{\htcsemanticsP(\thAt)}$ for every~$\thAt \in S$
  \\iff
  for every~$\thAt \in S$, there is~$w \in \V$ such that~$w \in \denAT{\thAt}$ and~$\restr{t}{\XAT} = w$
  \\iff
  for every~$\thAt \in S$, $\restr{t}{\XAT} \in \denAT{\thAt}$
  \\
  iff~$\restr{t}{\XAT} \in \denAT{S}$.
  \\
  Note that~$\vars{\thAt} \subseteq \XAT$ and, by condition~5 in Definition~\ref{def:structure}, ${\restr{t}{\XAT} = w}$ implies~$\restr{t}{\XAT}  \in \denAT{\thAt}$ iff~$w \in \denAT{\thAt}$.
\end{proof}

\begin{proof}[Proof of Proposition~\ref{prop:translation.satisfiable}]
    Since~$\AT$ is compositional, set~$S$ is \mbox{$\AT$-satisfiable}
    iff
    $$S \in \Sat = \{ S' \subseteq \TheoryAtoms \mid \den{S}_\AT \neq \emptyset \}$$
    iff
    there is a valuation~${w: \XAT \rightarrow \D_\AT}$ such that~${w \in \denAT{S}}$
    iff (Lemma~\ref{lem:translation.valuations})
    there is a valuation~${t: \X \rightarrow \Du}$ such that~${t \in \denAT{\htcsemanticsP(S)}}$
    and~${\restr{t}{\XAT} = w}$
    iff
    $S$ is \mbox{$\AT$-satisfiable}.\end{proof} \subsection{Proof of the Theorem~\ref{thm:first_translation}}

To pave the way between the transformation approach described in Definition~\ref{def:stable.model.simple} and the translation in~$\HTC$ described in Section~\ref{sec:htcsemantics},
we introduce a second translation into~$\HTC$ that is closer to the transformation approach.
The proof of the Main Theorem is then divided into two main lemmas: the first establishes the correspondence between the transformation approach described in  Definition~\ref{def:stable.model.simple} and this second translation, and the second establishes the correspondence in~$\HTC$ between both translations.

Let us start by describing this second translation that we denote~$\htcsemanticsP_2$.
The most relevant feature of this translation is that it decouples the generation of possible abstract theory solutions $S \subseteq \TheoryAtoms$ from the derivation of their atoms $\thAt \in S$ in the logic program.
In particular, rather than directly including constraints $\htcsemanticsA(\thAt)$ in the translation of program rules as done with \eqref{eq:htcsemantics:rule}, we use now a new, auxiliary propositional atom $\prop{\thAt}$ whose connection to the constraint $\htcsemanticsA(\thAt)$ is explicitly specified by using additional formulas~$\Phi(\TheoryAtoms)$.
In that way, the rules of the \tprogram~$P$ correspond now to an \HTC\ encoding of a regular, propositional logic~$\kappa(P)$.

Translation $\htcsemanticsP_2$ produces an \htctheory\ with signature $\tuple{\X_2,\D,\C_2}$ and denotation~$\den{\cdot}_2$ that extends the signature $\tuple{\X,\D,\C}$ and denotation $\den{\cdot}$ of $\htcsemanticsP$ as follows:
\begin{align}
\X_2 \ \ &= \ \ \X \cup \{ \, \thAt \mid \thAt \in \TheoryAtoms \, \}\\
\C_2 \ \ &= \ \ \C \ \cup \{ \, \thAt \mid \thAt \in \TheoryAtoms \, \}
\end{align}
that is, we extend the set of variables $\X$ with one fresh variable with the same name than each abstract theory atom $\thAt \in \TheoryAtoms$ and the constraints $\C$ with the corresponding propositional constraint atom $\thAt$.
The denotation $\den{\cdot}_2$ for the $\htcsemanticsP_2$ translation simply extends the denotation for $\htcsemanticsP$ by including the already seen fixed denotation for propositional atoms applied to the new constraints such that~$\den{\thAt}_2 \eqdef \{ v \in \V \mid v(\thAt) = \true \}$.
Furthermore, we define $\den{\true}_2=\V_{\X_2,\D}$.

The new \HTC-encoding $\htcsemanticsP_2(P,\AT,\TheoryAtomsES)$ is defined for a \tprogram\ $P$ over $\langle\Atoms,\TheoryAtoms,\TheoryAtomsES\rangle$ as before, but consists of three sets of formulas:
\begin{eqnarray}
      \begin{aligned}
\htcsemanticsP_2(P,\AT,\TheoryAtomsES)
\ \eqdef \
\kappa(P,\TheoryAtomsES)
\cup \htctsols(\TheoryAtoms) 
\cup \DOM{\AT, \Atoms\cup \TheoryAtoms }
      \end{aligned}
\label{f:tau2}
\end{eqnarray}
where~$\kappa(P,\TheoryAtomsES)$ is the regular program corresponding to the program in~\eqref{eq:transformationB.simple}, that is,
$\kappa(P,\TheoryAtomsES) \ \eqdef \ \kappa(P \cup\{ \, {\thAt \vee \comp{\thAt} } \, \mid \, \thAt\in \TheoryAtomsES \, \})$;
$\htctsols(\AT)$ consists of formulas
\begin{align}
      \prop{\thAt} \to \htcsemanticsA(\thAt) \hspace{1.5cm}&&\text{for every theory atom } \thAt \in \TheoryAtoms
      \label{eq:guess:choice}
\end{align}
We follow the definition of splitting sets for \HTC\ theories in~\citeN{cafascwa20b} (Definition~10). 
We can split translation $\htcsemanticsP_2(P,\AT,\TheoryAtomsES)$ into a bottom part $\kappa(P,\TheoryAtomsES) \cup \DOM{\AT, \Atoms\cup\TheoryAtoms }$
and a top part $\htctsols(\AT)$ via splitting set $\X_s \eqdef \Atoms \cup \TheoryAtoms$.
Then, for a valuation $v\in\V_{\X_2,\D}$, 
we define $E_{\X_s}(\htcsemanticsP_2(P,\AT,\TheoryAtomsES),v)$ as the theory
resulting from replacing all occurrences of variables in $\X_s$ in $\htctsols(\AT)$ by their values in $v$.

\begin{lemma}\label{lem:translation2.splitting}
      A valuation~$v$ is a stable model of~$\htcsemanticsP_2(P,\AT,\TheoryAtomsES)$ iff the following conditions hold:
      \begin{itemize}
            \item $\restr{v}{\X_s}$ is a stable model of~$\kappa(P,\TheoryAtomsES) \cup \DOM{\AT, \Atoms\cup\TheoryAtoms }$;
            \item $\restr{v}{\XAT}$ is a stable model of~$E_{\X_s}(\htctsols(\AT),v)$;
\end{itemize}
\end{lemma}

\begin{proof}
      By definition, variables in $\X_2\setminus\X_s=\XAT$ do not occur in $\kappa(P,\TheoryAtomsES) \cup \DOM{\AT, \Atoms\cup\TheoryAtoms }$,
and variables in $\X_s$ only occur in bodies in $\htctsols(\AT)$.
      Therefore, $\X_s$ is a splitting set.
      Then, we can directly apply Proposition~12 in~\citeN{cafascwa20b} 
      by instantiating $U=\X_s$, $\overline{U}=\XAT$, $\Pi=\htcsemanticsP_2(P,\AT,\TheoryAtomsES)$,
      $B_U(\Pi)=\kappa(P,\TheoryAtomsES) \cup \DOM{\AT, \Atoms\cup\TheoryAtoms }$, 
      and $E_U(\Pi,v) = E_{\X_s}(\htcsemanticsP_2(P,\AT,\TheoryAtomsES),v)$.
\end{proof}

\begin{lemma}\label{lem:translation2.stable.models.<=}
      Let~$t$ be a stable model of~$\htcsemanticsP_2(P,\AT,\TheoryAtomsES)$ and~$X = \{ \anyAt \in \Atoms \cup \TheoryAtoms \mid t(\prop{\anyAt}) = \true \}$.
Then, 
      \begin{itemize}
            \item $X$ is a~$\tuple{\AT,\TheoryAtomsES}$-stable model of a \tprogram\ $P$;
            \item $\restr{t}{\XAT} \in \denAT{X \cap \TheoryAtoms}$; and
            \item $t(x)$ is undefined for every~$x \in \XAT \setminus \vars{X \cap \TheoryAtoms}$.
      \end{itemize}
 \end{lemma}
 
 \begin{proof}
       By Lemma~\ref{lem:translation2.splitting},
       $t$ is a stable model of~$\htcsemanticsP_2(P,\AT,\TheoryAtomsES)$ iff the following conditions hold:
 \begin{enumerate}
      \item $\restr{t}{\X_s}$ is a stable model of~$\kappa(P,\TheoryAtomsES) \cup \DOM{\AT, \Atoms\cup\TheoryAtoms }$;
      \item $\restr{t}{\XAT}$ is a stable model of~$E_{\X_s}(\htctsols(\AT),t)$;
\end{enumerate}
 Furthermore, $\restr{t}{\X_s}$ is a stable model of~$\kappa(P,\TheoryAtomsES) \cup \DOM{\AT, \Atoms\cup\TheoryAtoms }$ iff~$X$ is a regular stable model of program~\eqref{eq:transformationB.simple}.
We show now that~$\restr{t}{\XAT} \in \denAT{X \cap \TheoryAtoms}$, which implies that~$(X \cap \TheoryAtoms) \in \Sat$ and, thus, that $X$ is a~$\tuple{\AT,\TheoryAtomsES}$-stable model~$P$.
By definition of~$X$, for every~$\thAt \in (X \cap \TheoryAtoms)$, it follows that~${t \in \den{\prop{\thAt}}}$ and, thus, ${\restr{t}{\X_s} \in \den{ \thAt}}$.
This implies that~${t \in \den{X \cap \TheoryAtoms}}$ and, thus, after some minor simplifications, that~$E_{\X_s}(\htcsemanticsP_2(P,\AT,\TheoryAtomsES),v)$ is equivalent to~$\htcsemanticsA(X \cap \TheoryAtoms)$.
Hence, ${t \in \den{\htcsemanticsA(X \cap \TheoryAtoms)}}$ and, by Lemma~\ref{lem:translation.valuations}, this implies that~$\restr{t}{\XAT} \in \denAT{X \cap \TheoryAtoms}$.
It remains to be shown that~$t(x)$ is undefined for every~$x \in \XAT \setminus \vars{X \cap \TheoryAtoms}$.
Pick any~$x \in \XAT \setminus \vars{X \cap \TheoryAtoms}$ and let~$h$ be a valuation such that~$h(y) = \restr{t}{\XAT}(y)$ for every~$y \in \X \setminus \{x\}$ and~$h(x) = \undefined$.
Since~$x \notin \vars{X \cap \TheoryAtoms}$, $h$ and~$t$ agree on all variables in~$\htcsemanticsA(X \cap \TheoryAtoms)$ and, thus, $h \in \den{\htcsemanticsA(X \cap \TheoryAtoms)}$.
Clearly~$h \subseteq \restr{t}{\XAT}$ and, since~$\restr{t}{\XAT}$ is a stable model of~$\htcsemanticsA(X \cap \TheoryAtoms)$, it follows that~$h = t$ and, thus, that~$t(x)$ is undefined.
\end{proof}

\begin{lemma}\label{lem:translation2.stable.models.=>}
      Let~$X \subseteq \Atoms \cup \TheoryAtoms$ be a~$\tuple{\AT,\TheoryAtomsES}$-stable model of a \tprogram\ $P$, 
      $v \in \denAT{X \cap \TheoryAtoms}$ be an assignment
      and~$t$ be a valuation such that~$t(\prop{\anyAt}) = \true$ iff~$b \in X$ for every~$\anyAt \in \Atoms \cup \TheoryAtoms$ and~$t(x) = \restr{v}{\Sigma}(x)$ for every~$x \in \XAT$ with~$\Sigma = \vars{X \cap \TheoryAtoms}$.
Then, $t$ is a stable model~$t$ of~$\htcsemanticsP_2(P,\AT,\TheoryAtomsES)$.
\end{lemma}

\begin{proof}
Assume that~$X$ is~$\tuple{\AT,\TheoryAtomsES}$-stable model of a \tprogram\ $P$.
Then, by Definition~\ref{def:stable.model.simple}, set~$X$ is a regular stable model of program~\eqref{eq:transformationB.simple} and~$(X \cap \TheoryAtoms) \in \Sat$.
By definition, the former implies that there is a stable model~$w$ of~$\kappa(P,\TheoryAtomsES)$ such that~$X = \{ \anyAt \in \Atoms \cup \TheoryAtoms \mid w(\prop{\anyAt}) = \true \}$.
Then,
\begin{itemize}
\item $\restr{t}{\X_s} = w$ is a stable model of~$\kappa(P,\TheoryAtomsES) \cup \DOM{\AT, \Atoms\cup\TheoryAtoms }$;
\item  after some minor simplifications, $E_{\X_s}(\htcsemanticsP_2(P,\AT,\TheoryAtomsES),v)$ is equivalent to~$\htcsemanticsA(X \cap \TheoryAtoms)$.
\end{itemize}
We show now that~${\restr{t}{\XAT} = \restr{v}{\Sigma}}$ is a stable model of~$E_{\X_s}(\htcsemanticsP_2(P,\AT,\TheoryAtomsES),v)$, that is, an stable model of~${\htcsemanticsA(X \cap \TheoryAtoms)}$.
Once this is shown, by Lemma~\ref{lem:translation2.splitting}, it follows that
$t$ is a stable model of~$\htcsemanticsP_2(P,\AT,\TheoryAtomsES)$.
Let~$t'$ be a valuation such that~$t'(x) = v(x)$ for every~$x \in \XAT$.
Then, $\restr{t'}{\XAT} = v$ and, by Lemma~\ref{lem:translation.valuations} and the assumption that~${v \in \denAT{X \cap \TheoryAtoms}}$, it follows that~$t' \in \den{\htcsemanticsA(X \cap \TheoryAtoms)}$.
Since~$t$ and~$t'$ agree on all variables occurring in~$\htcsemanticsA(X \cap \TheoryAtoms)$, it follows that~$t \in \den{X \cap \TheoryAtoms}$.
Furthermore, no valuation~${h \subset \restr{t}{\XAT}}$ satsifes~$h \in \den{\htcsemanticsA(X \cap \TheoryAtoms)}$ because all variables occurring in~$\htcsemanticsA(X \cap \TheoryAtoms)$ must be defined.
Hence, $\restr{t}{\XAT}$ is a stable model of~$E_{\X_s}(\htcsemanticsP_2(P,\AT,\TheoryAtomsES),v)$.
\end{proof}

\begin{lemma}\label{lem:translation2.answer.set.<=}
      Let~$t$ be stable model of~$\htcsemanticsP_2(P,\AT,\TheoryAtomsES)$.
Then, $(Y,v)$ is an answer set of~$P$ with~$Y = \{ \regAt \in \Atoms \mid t(\prop{\regAt}) = \true \}$ and~$v = \restr{t}{\XAT}$.
\end{lemma}

\begin{proof}
      By Lemma~\ref{lem:translation2.stable.models.<=},
      \begin{itemize}
            \item $X = \{ \anyAt \in \Atoms \cup \TheoryAtoms \mid t(\prop{\anyAt}) = \true \}$ is a~$\tuple{\AT,\TheoryAtomsES}$-stable model of~$P$;
            \item $\restr{t}{\XAT} \in \denAT{X \cap \TheoryAtoms}$; and
            \item $t(x)$ is undefined for every~$x \in \XAT \setminus \vars{X \cap \TheoryAtoms}$.
      \end{itemize}
By definition, this implies that~$(Y,\restr{v}{\Sigma})$ is an answer set of~$P$ with
      $$Y = X \cap \Atoms = \{ \regAt \in \Atoms \mid t(\prop{\regAt}) = \true \}$$
      and~$\Sigma = \vars{X \cap \TheoryAtoms}$.
Finally, note that the last item above implies~$v = \restr{v}{\Sigma}$ and the result holds.
\end{proof}

\begin{lemma}\label{lem:translation2.answer.set.=>}
      Let~$(Y,v)$ be an answer set of~$P$,
\begin{eqnarray*}      
      X & = &  Y
      \, \cup \, \big\{ \ \thAt\in \TheoryAtomsES  \mid  v \IN \denAT{\thAt} \ \big\}\\
      & &  \cup \, \big\{\, \thAt \in \TheoryAtomsDN \mid (B \to \thAt) \in P \text{ and } (Y,v) \models B \ \big\}
\end{eqnarray*}

      and~$t$ be a valuation such that~${t(\prop{\anyAt}) = \true}$ iff~$b \in X$ for every~$\anyAt \in \Atoms \cup \TheoryAtoms$ and~$t(x) = v(x)$ for every~$x \in \XAT$.
Then, $t$ is a stable model of~$\htcsemanticsP_2(P,\AT,\TheoryAtomsES)$
\end{lemma}

\begin{proof}
      Let~$(Y,v)$ be an answer set of~$P$.
By Proposition~\ref{prop:answer.sets.correspondence},
      it follows that~$X$
is a~$\tuple{\AT,\TheoryAtomsES}$-stable model of~$P$ and that there is a~$w \in \denAT{X \cap \TheoryAtoms}$ such that~$v = \restr{w}{\Sigma}$ with~$\Sigma = \vars{X \cap \TheoryAtoms}$.
By Lemma~\ref{lem:translation2.stable.models.=>}, this implies that~$t$ is a stable model.
\end{proof}

\begin{lemma}\label{lem:translation2.implied.formulas}
Every \HTC-model of~$\htcsemanticsP_2(P,\AT,\TheoryAtomsES)$ satisfies
\begin{align}
      \thAt \wedge \comp{\thAt}  &\to  \bot && \text{for each } \thAt \in \TheoryAtoms \label{eq:htctrans:consistent}
      \\
      \htcsemanticsA(\thAt)  &\leftrightarrow  \prop{\thAt} && \text{for each } \thAt \in \TheoryAtomsES \label{eq:htctrans:external_equivalence}
\end{align}
\end{lemma}

\begin{proof}
Pick an \HTC-model~$\tuple{h,t}$ of~$\htcsemanticsP_2(P,\AT,\TheoryAtomsES)$.
We show first that it satisfies~\eqref{eq:htctrans:consistent}.
Pick a theory atom~$\thAt \in \TheoryAtoms$. 
If~$t \notin \den{\thAt}$, then~$\tuple{h,t}$ satisfies the implication in~\eqref{eq:htctrans:consistent} and the result holds.
Hence, assume without loss of generality that~${t \in \den{\thAt}}$.
Since~$\prop{\thAt} \to  \htcsemanticsA(\thAt)$ belongs to~$\htcsemanticsP_2(P,\AT,\TheoryAtomsES)$, it follows that~$\tuple{h,t} \not\models \neg\htcsemanticsA(\thAt)$ and, thus, $t \in \den{\htcsemanticsA(\thAt)}$.
Since~$\AT$ is consistent, this implies that~$t \notin \den{\htcsemanticsA(\comp{\thAt})}$ and, thus, that~$\tuple{h,t} \models \neg\htcsemanticsA(\comp{\thAt})$.
Since~$\prop{\comp{\thAt}} \to  \htcsemanticsA(\comp{\thAt})$ belongs to~$\htcsemanticsP_2(P,\AT,\TheoryAtomsES)$, it follows that~$\tuple{h,t} \not\models \prop{\comp{\thAt}}$ and, thus, $\tuple{h,t}$ satisfies the implication in~\eqref{eq:htctrans:consistent} and the result holds.
\\[5pt]
Let us now show that~$\tuple{h,t}$ satisfies~\eqref{eq:htctrans:external_equivalence}.
Pick a theory atom~$\thAt \in \TheoryAtomsES$.
Note that~$\prop{\thAt} \to \htcsemanticsA(\thAt)$ belongs to~$\htcsemanticsP_2(P,\AT,\TheoryAtomsES)$ and, thus, it only remains to show that~$\tuple{h,t}$ satisfies~$\htcsemanticsA(\thAt) \to \prop{\thAt}$.
Furthermore, $\thAt \vee \comp{\thAt}$ belongs to~$\htcsemanticsP_2(P,\AT,\TheoryAtomsES)$ and, by the first part of the proof, $\tuple{h,t}$ satisfies~$\thAt \wedge \comp{\thAt} \to  \bot $.
Hence one of the following two cases holds:
\begin{itemize}
      \item $h(\prop{\thAt}) = t(\prop{\thAt}) = \true$; or
\item
$h(\comp{\thAt}) = t(\comp{\thAt}) = \true$.
\end{itemize}
If the former holds, then~$\tuple{h,t}$ satisfies~$\htcsemanticsA(\thAt) \to  \prop{\thAt}$ because it satisfies its consequent in both~$h$ and~$t$.
Otherwise,
$h(\comp{\thAt}) = t(\comp{\thAt}) = \true$.
Since~$\neg \htcsemanticsA(\comp{\thAt}) \wedge \prop{\comp{\thAt}} \to  \bot$ belongs to~$\htcsemanticsP_2(P,\AT,\TheoryAtomsES)$, it follows that~$\tuple{h,t} \not\models \neg\htcsemanticsA(\comp{\thAt})$ and, thus, $t \in \den{\htcsemanticsA(\comp{\thAt})}$.
Since~$\AT$ is consistent, this implies that~$t \notin \den{\htcsemanticsA(\thAt)}$ and, thus, $\tuple{h,t}$ satisfies~$\htcsemanticsA(\thAt) \to  \prop{\thAt}$ because it does not satisfy its antecedent.
\end{proof}

\noindent
Let~$\htcsemanticsP_3(P,\AT,\TheoryAtomsES)$
be the theory obtained from~$\htcsemanticsP_2(P,\AT,\TheoryAtomsES)$ by 
\begin{itemize}
      \item replacing each occurrence of~$\prop{\thAt}$ with~$\htcsemanticsA(\thAt)$ for every~$\thAt \in \TheoryAtomsES$;
      
      \item adding~$\htcsemanticsA(\thAt) \leftrightarrow \prop{\thAt}$ for every~$\thAt \in \TheoryAtomsES$.
      
      \item adding~$\htcsemanticsA^B(r) \to \htcsemanticsA(b_0)$ for every~$r \in P$ of the form of~\eqref{theory:rule}.
\end{itemize}

\begin{lemma}\label{lem:translation23}
$\htcsemanticsP_2(P,\AT,\TheoryAtomsES)$ and~$\htcsemanticsP_3(P,\AT,\TheoryAtomsES)$ have the same \HTC-models.
\end{lemma}

\begin{proof}
      By Lemma~\ref{lem:translation2.implied.formulas}, adding adding~$\htcsemanticsA(\thAt) \leftrightarrow \prop{\thAt}$ for every~$\thAt \in \TheoryAtomsES$ does not change the \HTC-models of~$\htcsemanticsP_2(P,\AT,\TheoryAtomsES)$.
Furthermore, in the presence of these equivalences, replacing each occurrence of~$\prop{\thAt}$ with~$\htcsemanticsA(\thAt)$ for every~$\thAt \in \TheoryAtomsES$ does not change the \HTC-models either.
It remains to show that adding~$\htcsemanticsA^B(r) \to \htcsemanticsA(b_0)$ for every~$r \in P$ of the form of~\eqref{theory:rule} does not change the \HTC-models.
Let~$\Gamma$ be the result of the two first steps.
As we just showed, the \HTC-models of~$\htcsemanticsP_2(P,\AT,\TheoryAtomsES)$ are the same as the \HTC-models of~$\Gamma$.
If~$b_0 \in \Atoms \cup \TheoryAtomsES$, then this implication is already present in~$\Gamma$.
Otherwise, $b_0 \in \TheoryAtomsDN$ and there are implications~$\htcsemanticsA^B(r) \to b_0$ and~$b_0 \to \htcsemanticsA(b_0)$ in~$\Gamma$.
Hence, the implication~$\htcsemanticsA^B(r) \to \htcsemanticsA(b_0)$ is satisfied by every \HTC-model of~$\Gamma$, which are the same as the \HTC-models of~$\htcsemanticsP_2(P,\AT,\TheoryAtomsES)$.
\end{proof}

\noindent
Let~$\htcsemanticsP_4(P,\AT,\TheoryAtomsES)$
be the theory obtained from~$\htcsemanticsP_3(P,\AT,\TheoryAtomsES)$ by removing~$\DOM{\TheoryAtomsES}$ and formulas~$\htcsemanticsA(\thAt) \leftrightarrow \prop{\thAt}$ for every~$\thAt \in \TheoryAtomsES$.

\begin{lemma}\label{lem:translation34.htmodels}
Let~$\tuple{h,t}$ and~$\tuple{h',t'}$ be two \HTC-interpretation such that~$h' = \restr{h}{X}$ and $t' = \restr{t}{X}$ with~$X = \X_2 \setminus \TheoryAtomsES$; and
\begin{align*}
      t(\prop{\thAt}) = \true \quad &\text{iff}\quad t' \in \den{\htcsemanticsA(\thAt)} &&&
      h(\prop{\thAt}) = \true \quad &\text{iff}\quad h' \in \den{\htcsemanticsA(\thAt)}
\end{align*}
for every~$\thAt \in \TheoryAtomsES$.
Then, $\tuple{h,t}$ is an \HTC-model of~$\htcsemanticsP_3(P,\AT,\TheoryAtomsES)$ iff~$\tuple{h',t'}$ is an \HTC-model of~$\htcsemanticsP_4(P,\AT,\TheoryAtomsES)$.
\end{lemma}

\begin{proof}
      Since no $\thAt \in \TheoryAtomsES$ occurs in~$\htcsemanticsP_4(P,\AT,\TheoryAtomsES)$ and $\tuple{h,t}$ and~$\tuple{h',t'}$ agree on all other variables it only needs to be shown that~$\tuple{h,t}$ satisfies~$\DOM{\TheoryAtomsES}$ and~$\htcsemanticsA(\thAt) \leftrightarrow \prop{\thAt}$ for every~$\thAt \in \TheoryAtomsES$, which is easy to see by construction.
\end{proof}

\noindent
Let~$\htcsemanticsP_5(P,\AT,\TheoryAtomsES)$
be the theory obtained from~$\htcsemanticsP_4(P,\AT,\TheoryAtomsES)$ by removing~$\DOM{\TheoryAtomsDN}$ and all implications~$\htcsemanticsA^B(r) \to \prop{\thAt}$ and~$\prop{\thAt} \to \htcsemanticsA(\thAt)$ with~$\thAt \in \TheoryAtomsDN$.

\begin{lemma}\label{lem:translation45.htmodels}
Let~$\tuple{h,t}$ and~$\tuple{h',t'}$ be two \HTC-interpretation such that~$h' = \restr{h}{X}$ and~$t' = \restr{t}{X}$ with~$X=\X_2 \setminus \TheoryAtomsDN$; and
\begin{align*}
t(\prop{\thAt}) = \true \quad &\text{iff}\quad t' \models \htcsemanticsA^B(r)
\quad \text{for some } r \in P \text{ with head } \thAt\\
h(\prop{\thAt}) = \true \quad &\text{iff}\quad \tuple{h',t'} \models \htcsemanticsA^B(r)
\quad \text{for some } r \in P \text{ with head } \thAt
\end{align*}
for every~$\thAt \in \TheoryAtomsDN$.
Then, $\tuple{h,t}$ is an \HTC-model of~$\htcsemanticsP_4(P,\AT,\TheoryAtomsES)$ iff~$\tuple{h',t'}$ is an \HTC-model of~$\htcsemanticsP_5(P,\AT,\TheoryAtomsES)$.
\end{lemma}

\begin{proof}
      Since~$\tuple{h,t}$ and~$\tuple{h',t'}$ agree on all variables occurring in~$\htcsemanticsP_5(P,\AT,\TheoryAtomsES)$, it only needs to be shown that~$\tuple{h,t}$ satisfies~$\htcsemanticsA^B(r) \to \prop{\thAt}$ and~$\prop{\thAt} \to \htcsemanticsA(\thAt)$ for every~$\thAt \in \TheoryAtomsDN$.
By construction, $\tuple{h,t}$ satisfies~$\DOM{\TheoryAtomsDN}$ and~$\htcsemanticsA^B(r) \to \prop{\thAt}$.
Pick a rule of the form~$\prop{\thAt} \to \htcsemanticsA(\thAt)$.
We proceed by cases.
\emph{Case 1}. There is no rule of the form~$\htcsemanticsA^B(r) \to \prop{\thAt}$ in~$\htcsemanticsP_4(P,\AT,\TheoryAtomsES)$.
Then~$h(\prop{\thAt}) = t(\prop{\thAt}) = \undefined$ and, thus, $\tuple{h,t}$ satisfies~$\prop{\thAt} \to \htcsemanticsA(\thAt)$.
\emph{Case 2}. There is a rule of the form~$\htcsemanticsA^B(r) \to \prop{\thAt}$ in~$\htcsemanticsP_4(P,\AT,\TheoryAtomsES)$.
Hence, rule~${\htcsemanticsA^B(r) \to \htcsemanticsA(\thAt)}$ belongs to~$\htcsemanticsP_4(P,\AT,\TheoryAtomsES)$ and, thus, $\tuple{h,t} \models \prop{\thAt}$ implies that~$\tuple{h,t} \models \htcsemanticsA^B(r)$ for some~$r \in P$ with head~$\thAt$ and, thus, that~$\tuple{h,t} \models \htcsemanticsA(\thAt)$.
Hence, $\tuple{h,t}$ satisfies~${\prop{\thAt} \to \htcsemanticsA(\thAt)}$.
\end{proof}

\begin{lemma}\label{lem:translation35.htmodels}
      Let~$\tuple{h,t}$ and~$\tuple{h',t'}$ be two \HTC-interpretation such that~$h' = \restr{h}{X}$ and~$t' = \restr{t}{X}$ with~$X = \X_2 \setminus \TheoryAtoms$;
      \begin{align*}
            t(\prop{\thAt}) = \true \quad &\text{iff}\quad t' \in \den{\htcsemanticsA(\thAt)} &&&
            h(\prop{\thAt}) = \true \quad &\text{iff}\quad h' \in \den{\htcsemanticsA(\thAt)}
      \end{align*}
      for every~$\thAt \in \TheoryAtomsES$; and
      \begin{align*}
      t(\prop{\thAt}) = \true \quad &\text{iff}\quad t' \models \htcsemanticsA^B(r)
      \quad \text{for some } r \in P \text{ with head } \thAt\\
      h(\prop{\thAt}) = \true \quad &\text{iff}\quad \tuple{h',t'} \models \htcsemanticsA^B(r)
      \quad \text{for some } r \in P \text{ with head } \thAt
      \end{align*}
      for every~$\thAt \in \TheoryAtomsDN$.
Then, $\tuple{h,t}$ is an \HTC-model of~$\htcsemanticsP_3(P,\AT,\TheoryAtomsES)$ iff~$\tuple{h',t'}$ is an \HTC-model of~$\htcsemanticsP_5(P,\AT,\TheoryAtomsES)$.
\end{lemma}

\begin{proof}
      Let~$\tuple{h'',t''}$ be an \HTC-interpretation such that~$h'' = \restr{h}{Y}$ and~$t'' = \restr{t}{Y}$ with~$Y = \X_2 \setminus \TheoryAtomsES$;
      \begin{align*}
            t(\prop{\thAt}) = \true \quad &\text{iff}\quad t' \in \den{\htcsemanticsA(\thAt)} &&&
            h(\prop{\thAt}) = \true \quad &\text{iff}\quad h' \in \den{\htcsemanticsA(\thAt)}
      \end{align*}
      for every~$\thAt \in \TheoryAtomsES$.
By Lemma~\ref{lem:translation34.htmodels}, $\tuple{h,t}$ is an \HTC-model of~$\htcsemanticsP_3(P,\AT,\TheoryAtomsES)$ iff~$\tuple{h'',t''}$ is an \HTC-model of~$\htcsemanticsP_4(P,\AT,\TheoryAtomsES)$.
By Lemma~\ref{lem:translation45.htmodels}, $\tuple{h'',t''}$ is an \HTC-model of~$\htcsemanticsP_4(P,\AT,\TheoryAtomsES)$ iff~$\tuple{h',t'}$ is an \HTC-model of~$\htcsemanticsP_5(P,\AT,\TheoryAtomsES)$.
\end{proof}

\begin{lemma}\label{lem:translation35.stable.models}
      Let~$t$ and~$t'$ be two valuations such that~$t' = \restr{t}{X}$ with~$X = \X_2 \setminus\TheoryAtoms$;
      $t(\prop{\thAt}) = \true$ iff~$t' \in \den{\htcsemanticsA(\thAt)}$ for every~$\thAt \in \TheoryAtomsES$; and
      $t(\prop{\thAt}) = \true$ iff~$t' \models \htcsemanticsA^B(r)$ for some~$r \in P$ with head~$\thAt$ for every~$\thAt \in \TheoryAtomsDN$.
Then, $t$ is a stable model of~$\htcsemanticsP_3(P,\AT,\TheoryAtomsES)$ iff~$t'$ is a stable model of~$\htcsemanticsP_5(P,\AT,\TheoryAtomsES)$.
\end{lemma}

\begin{proof}
      \emph{Left to right.}
Assume that~$t$ is a stable model of~$\htcsemanticsP_3(P,\AT,\TheoryAtomsES)$.
Then, by Lemma~\ref{lem:translation35.htmodels}, $t'$ satisfies~$\htcsemanticsP_4(P,\AT,\TheoryAtomsES)$.
Pick a valuation~$h' \subset t'$ and let~$h$ be a valuation such tha~$h' = \restr{h}{X}$;
      $h(\prop{\thAt}) = \true$ iff~$h' \in \den{\htcsemanticsA(\thAt)}$ for every~$\thAt \in \TheoryAtomsES$;
      and~$h(\prop{\thAt}) = \true$ iff~${h' \models \htcsemanticsA^B(r)}$ for some~$r \in P$ with head~$\thAt$ for every~$\thAt \in \TheoryAtomsDN$.
Then, $h \subset t$ and, since~$t$ is a stable model of~$\htcsemanticsP_3(P,\AT,\TheoryAtomsES)$, it follows that~$\tuple{h,t}$ does not satisfy~$\htcsemanticsP_3(P,\AT,\TheoryAtomsES)$.
By Lemma~\ref{lem:translation35.htmodels}, this implies that~$\tuple{h',t'}$ does not satisfy~$\htcsemanticsP_5(P,\AT,\TheoryAtomsES)$.
Hence, no \HTC-interpretation~$\tuple{h',t'}$ with~$h' \subset t'$ satisfies~$\htcsemanticsP_5(P,\AT,\TheoryAtomsES)$ and, thus, $t'$ is a stable model of~$\htcsemanticsP_5(P,\AT,\TheoryAtomsES)$.
      \\[5pt]
      \emph{Right to left.}
Assume that~$t'$ is a stable model of~$\htcsemanticsP_5(P,\AT,\TheoryAtomsES)$.
Then, by Lemma~\ref{lem:translation35.htmodels}, $t$ satisfies~$\htcsemanticsP_3(P,\AT,\TheoryAtomsES)$.
Pick a valuation~$h \subset t$ and let~$h' = \restr{h}{X}$.
We proceed by cases.
\emph{Case 1}. $h' \subset t'$. Then, since~$t'$ is a stable model of~$\htcsemanticsP_5(P,\AT,\TheoryAtomsES)$, it follows that~$\tuple{h',t'}$ does not satisfy~$\htcsemanticsP_5(P,\AT,\TheoryAtomsES)$.
By Lemma~\ref{lem:translation35.htmodels}, this implies that~$\tuple{h,t}$ does not satisfy~$\htcsemanticsP_3(P,\AT,\TheoryAtomsES)$.
\emph{Case 2}. $h' = t'$. Since~$h \subset t$ this implies that there is~$\prop{\thAt} \in \TheoryAtoms$ such that~$h(\prop{\thAt}) = \undefined$ and~$t(\prop{\thAt}) = \true$.
\emph{Case 2.1}. $\thAt \in \TheoryAtomsES$. Then, by assumption, $t' \in \den{\htcsemanticsA(\thAt)}$ and, since~$\restr{h}{X} = h' = t' = \restr{t}{X}$, this implies that~$\{h,t\} \subseteq \den{\htcsemanticsA(\thAt)}$.
Hence, $\tuple{h,t}$ does not satisfy~$\htcsemanticsA(\thAt) \leftrightarrow \prop{\thAt}$ and, thus, it does not satisfy~$\htcsemanticsP_3(P,\AT,\TheoryAtomsES)$.
\emph{Case 2.2}. $\thAt \in \TheoryAtomsDN$. 
Then, by assumption, $t' \models \htcsemanticsA^B(r)$ for some~$r \in P$ with head~$\thAt$ and, since~$\restr{h}{X} = h' = t' = \restr{t}{X}$, this implies  that~$\tuple{h,t} \models \htcsemanticsA^B(r)$.
Hence, $\tuple{h,t}$ does not satisfy~$\htcsemanticsA^B(r) \to \prop{\thAt}$ and, thus, it does not satisfy~$\htcsemanticsP_3(P,\AT,\TheoryAtomsES)$.
Hence, no \HTC-interpretation~$\tuple{h,t}$ with~$h \subset t$ satisfies~$\htcsemanticsP_3(P,\AT,\TheoryAtomsES)$ and, thus, $t$ is a stable model of~$\htcsemanticsP_3(P,\AT,\TheoryAtomsES)$.
\end{proof}

\begin{lemma}\label{lem:translation15.htmodels}
      $\htcsemanticsP(P,\AT,\TheoryAtomsES)$ and~$\htcsemanticsP_5(P,\AT,\TheoryAtomsES)$ have the same \HTC-models.
\end{lemma}

\begin{proof}
Let
\begin{eqnarray*}
\Gamma & = & \{ \
      \htcsemanticsA^B(r) \rightarrow \htcsemanticsA(b_0) \mid r \in P \text{ with $b_0$ the head of $r$}\
      \ \} \nonumber \\
      & & \cup \ \DOM{\AT,\Atoms} 
\end{eqnarray*}
Then, $\htcsemanticsP(P,\AT,\TheoryAtomsES)$ is the result of adding
$$
x = x
\text{ for every } x \in \vars{\thAt} \text{ and }\thAt \in \TheoryAtomsES.
$$
to~$\Gamma$ and~$\htcsemanticsP_5(P,\AT,\TheoryAtomsES)$ is the result of adding
$$
\htcsemanticsA(\thAt) \vee \htcsemanticsA(\comp{\thAt}) \qquad
\text{ for every } \thAt \in \TheoryAtomsES.
$$
to~$\Gamma$.
Pick any~$\thAt \in \TheoryAtomsES$ and any \HTC-interpretation~$\tuple{h,t}$.
Then~$\vars{\thAt} = \vars{\comp{\thAt}}$ and~$\tuple{h,t} \models \htcsemanticsA(\thAt) \vee \htcsemanticsA(\comp{\thAt})$ iff either~$h \in \den{\htcsemanticsA(\thAt)}$ or~$h \in \den{\htcsemanticsA(\comp{\thAt})}$.
This implies that~$h(x)$ is defined for every~$x \in \vars{\thAt}$ and, thus, that~$\tuple{h,t} $ satsifies~$x = x$.
Conversely, if~$\tuple{h,t}$ satisfies~$x = x$ for every~$x \in \vars{\thAt}$, then~$h \in \den{\htcsemanticsA(\thAt)}$ or~$h \in \den{\htcsemanticsA(\comp{\thAt})}$ because the complement is absolute.
Hence, $\htcsemanticsP(P,\AT,\TheoryAtomsES)$ and~$\htcsemanticsP_5(P,\AT,\TheoryAtomsES)$ have the same \HTC-models.
\end{proof}

\begin{lemma}\label{lem:main.correspondence}
If~$(Y,v)$ is an~$\tuple{\AT,\TheoryAtomsES}$-answer set of~$P$, then~$t = v \cup \{ \regAt \mapsto \true \mid \regAt \in Y \}$ is a stable model of~$\htcsemanticsP(P,\AT,\TheoryAtomsES)$.
Conversely,
if~$t$ is a stable model of~$\htcsemanticsP(P,\AT,\TheoryAtomsES)$, then~$(Y,v)$ is an~$\tuple{\AT,\TheoryAtomsES}$-answer set of~$P$ with~$Y = \{ \regAt \in \Atoms \mid t(\prop{\regAt}) = \true \}$ and~$v = \restr{t}{\XAT}$. 
\end{lemma}

\begin{proof}
      Assume that~$(Y,v)$ is an~$\tuple{\AT,\TheoryAtomsES}$-answer set of~$P$.
Let
\begin{eqnarray*}
      X & = &  Y
      \, \cup \, \big\{ \ \thAt\in \TheoryAtomsES  \mid  v \in \denAT{\thAt} \ \big\}\\
      & & \cup \big\{\, \thAt \in \TheoryAtomsDN \mid (B \to \thAt) \in P \text{ and } (Y,v) \models B \ \big\} 
\end{eqnarray*}
      and~$t' = v \cup \{ \anyAt \mapsto \true \mid \anyAt \in X \}$.
Then, by Lemma~\ref{lem:translation2.answer.set.=>}, it follows that~$t'$ is a stable model of~$\htcsemanticsP_2(P,\AT,\TheoryAtomsES)$.
By Lemma~\ref{lem:translation23}, this implies that~$t'$ is a stable model of~$\htcsemanticsP_3(P,\AT,\TheoryAtomsES)$.
By Lemma~\ref{lem:translation35.stable.models}, this implies that~$t = \restr{t'}{Z}$ is a stable model of~$\htcsemanticsP_5(P,\AT,\TheoryAtomsES)$ with~${Z = \X_2 \setminus\TheoryAtoms}$.
By Lemma~\ref{lem:translation15.htmodels}, this implies that~$t$ is a stable model of~$\htcsemanticsP(P,\AT,\TheoryAtomsES)$.
      \\[5pt]
      Conversely, assume that~$t$ is a stable model of~$\htcsemanticsP(P,\AT,\TheoryAtomsES)$.
By Lemma~\ref{lem:translation15.htmodels}, this implies that~$t$ is a stable model of~$\htcsemanticsP_5(P,\AT,\TheoryAtomsES)$.
By Lemma~\ref{lem:translation35.stable.models}, this implies that~$t'$ is a stable model of~$\htcsemanticsP_3(P,\AT,\TheoryAtomsES)$ with~$t'$ defined as follows:
      \begin{itemize}
            \item $\restr{t'}{Z} = t$ with~$Z = \X_2 \setminus\TheoryAtoms$;
            \item $t'(\prop{\thAt}) = \true$ iff~$t \in \den{\htcsemanticsA(\thAt)}$ for every~$\thAt \in \TheoryAtomsES$;
            \item $t'(\prop{\thAt}) = \true$ iff~$t \models \htcsemanticsA^B(r)$ for some~$r \in P$ with head~$\thAt$ for every~$\thAt \in \TheoryAtomsDN$.
      \end{itemize}
      By Lemma~\ref{lem:translation35.htmodels}, this implies that~$t'$ is a stable model of~$\htcsemanticsP_2(P,\AT,\TheoryAtomsES)$ and, by Lemma~\ref{lem:translation2.answer.set.<=}, that~$(Y,v)$ is an~$\tuple{\AT,\TheoryAtomsES}$-answer set of~$P$.
Note that
      $$Y = \{ \regAt \in \Atoms \mid t(\prop{\regAt}) = \true \} = \{ \regAt \in \Atoms \mid t'(\prop{\regAt}) = \true \}$$ and~$v = \restr{t}{\XAT} = \restr{t'}{\XAT}$.
\end{proof}

\begin{proof}[Proof of the Theorem]
      The correspondence is established by Lemma~\ref{lem:main.correspondence}.
To see that it is one-to-one, note that the construction of~$t$ from~$(Y,v)$ is unique.
Similarly, $(Y,v)$ is uniquely determined by~$t$ defining~$Y = \{ \regAt \in \Atoms \mid t(\prop{\regAt}) = \true \}$ and~$v = \restr{t}{\XAT}$.
\end{proof}

\end{document}